\documentclass[accepted]{uai2025}
\usepackage[american]{babel}
\usepackage{natbib}
\usepackage{mathtools}
\usepackage{booktabs}
\usepackage{tikz}

\usepackage{algorithm}
\usepackage{algpseudocode}

\usepackage{amsmath}
\usepackage{amssymb}
\usepackage{mathtools}
\usepackage{amsthm}

\theoremstyle{plain}
\newtheorem{thm}{Theorem}
\newtheorem{prop}{Proposition}
\newtheorem{defn}{Definition}

\newcommand{\circDist}{\textsf{CW}}
\newcommand{\empCircDist}{\textsf{ECW}}
\newcommand{\gmmDist}{\textsf{MW}}
\newcommand{\wassDist}{\textsf{W}}

\newcommand{\data}{\mathcal{D}}

\newcommand{\lp}{\textsf{LP}}

\newcommand{\ch}{\textsf{ch}}
\newcommand{\scope}{\textsf{sc}}

\usepackage{xspace}
\usepackage{dsfont}
\newcommand{\rvars}[1]{\ensuremath{\mathbf{#1}}\xspace}
\newcommand{\X}{\rvars{X}}
\newcommand{\Y}{\rvars{Y}}

\newcommand{\jstate}[1]{\ensuremath{\mathbf{#1}}\xspace}
\newcommand{\x}{\jstate{x}}
\newcommand{\y}{\jstate{y}}

\newcommand{\norm}[1]{\left\lVert#1\right\rVert}

\newcommand{\ac}[1]{{\textcolor{black}{#1}}}

\newcommand{\binary}{\ensuremath{\{0,1\}}}

\DeclareMathOperator{\Ex}{\mathbb{E}}

\newcommand{\rethm}[3]{\newtheorem*{#1}{Theorem~\ref{#1}}
\begin{#1}[#2]#3\end{#1}}

\newcommand{\reprop}[3]{\newtheorem*{#1}{Proposition~\ref{#1}}
\begin{#1}[#2]#3\end{#1}}
\usepackage{tikz}

\usetikzlibrary{pc}
\usetikzlibrary{positioning}

\newcommand{\midlinewidth}{1.0pt}

\newcommand{\middist}{17pt}

\newcommand{\halfdist}{19pt}

\definecolor{lacamlilac} {RGB} {107,93,153}
\definecolor{lacamlilac2} {RGB} {93, 109, 152}
\definecolor{lacamlightlilac} {RGB} {174, 166, 201}
\definecolor{lacamdarklilac} {RGB} {51, 10, 102}
\definecolor{lacamdarklilac5} {RGB} {51, 10, 102}
\colorlet{lacamdarklilac4} {lacamdarklilac5!80!}
\colorlet{lacamdarklilac3} {lacamdarklilac5!60!}
\colorlet{lacamdarklilac2} {lacamdarklilac5!40!}
\colorlet{lacamdarklilac1} {lacamdarklilac5!20!}

\definecolor{lacamgold5} {RGB} {255, 87, 0}
\colorlet{lacamgold4} {lacamgold5!80!}
\colorlet{lacamgold3} {lacamgold5!60!}
\colorlet{lacamgold2} {lacamgold5!40!}
\colorlet{lacamgold1} {lacamgold5!20!}
\definecolor{violet} {RGB} {119, 111, 178}
\definecolor{petroil2} {RGB} {36, 165, 175}
\definecolor{petroil4} {RGB} {30, 132, 149}
\definecolor{petroil6} {RGB} {23, 101, 115}
\definecolor{gold2} {RGB} {255, 130, 0}
\definecolor{gold4} {RGB} {250, 100, 0}
\definecolor{gold6} {RGB} {245, 90, 0}

\definecolor{lacamoil5}{rgb}{0.13, 0.67, 0.8}
\colorlet{lacamoil4} {lacamoil5!80!}
\colorlet{lacamoil3} {lacamoil5!60!}
\colorlet{lacamoil2} {lacamoil5!40!}
\colorlet{lacamoil1} {lacamoil5!20!}

\definecolor{tomato0} {HTML} {EF9A9A}
\definecolor{tomato1} {HTML} {F44336}
\definecolor{tomato2} {HTML} {E53935}
\definecolor{tomato3} {HTML} {D32F2F}
\definecolor{tomato4} {HTML} {C62828}
\definecolor{tomato5} {HTML} {B71C1C}

\definecolor{peas1} {HTML} {009688}
\definecolor{peas2} {HTML} {00897B}
\definecolor{peas3} {HTML} {00796B}
\definecolor{peas4} {HTML} {00695C}
\definecolor{peas5} {HTML} {004D40}

\definecolor{bgrey0} {HTML} {78909C}
\definecolor{bgrey1} {HTML} {607D8B}
\definecolor{bgrey2} {HTML} {546E7A}
\definecolor{bgrey3} {HTML} {455A64}
\definecolor{bgrey4} {HTML} {37474F}
\definecolor{bgrey5} {HTML} {263238}

\definecolor{olive0} {HTML} {C5E1A5}
\definecolor{olive1} {HTML} {AED581}
\definecolor{olive2} {HTML} {9CCC65}
\definecolor{olive3} {HTML} {8BC34A}
\definecolor{olive4} {HTML} {7CB342}
\definecolor{olive5} {HTML} {689F38}

\definecolor{pink0} {HTML} {FCE4EC}
\definecolor{pink1} {HTML} {F8BBD0}
\definecolor{pink2} {HTML} {F48FB1}
\definecolor{pink3} {HTML} {F06292}
\definecolor{pink4} {HTML} {EC407A}
\definecolor{pink5} {HTML} {FF80AB}

\definecolor{brown0} {HTML} {D7CCC8}
\definecolor{brown1} {HTML} {BCAAA4}
\definecolor{brown2} {HTML} {A1887F}
\definecolor{brown3} {HTML} {8D6E63}
\definecolor{brown4} {HTML} {795548}
\definecolor{brown5} {HTML} {6D4C41}
\definecolor{brown6} {HTML} {5D4037}

\definecolor{yellow0} {HTML} {CDDC39}
\definecolor{yellow1} {HTML} {9E9D24}
\definecolor{yellow3} {HTML} {FFBD2A}
\definecolor{yellow4} {HTML} {FFB000}
\definecolor{yellow5} {HTML} {FFD600}

\definecolor{ForestGreen}{RGB}{34,139,34}

\title{Optimal Transport for Probabilistic Circuits}

\author[1]{\href{mailto:<acioting@asu.edu>?Subject=Your Paper `Optimal Transport for Probabilistic Circuits`}{Adrian~Ciotinga}{}}
\author[1]{YooJung Choi}
\affil[1]{%
    School of Computing and Augmented Intelligence\\
    Arizona State University
}
  
\pgfplotsset{compat=1.18}
  
\begin{document}
\maketitle

\begin{abstract}
We introduce a novel optimal transport framework for probabilistic circuits (PCs). While it has been shown recently that divergences between distributions represented as certain classes of PCs can be computed tractably, to the best of our knowledge, there is no existing approach to compute the Wasserstein distance between probability distributions given by PCs. We propose a Wasserstein-type distance that restricts the coupling measure of the associated optimal transport problem to be a probabilistic circuit. We then develop an algorithm for computing this distance by solving a series of small linear programs and derive the circuit conditions under which this is tractable. Furthermore, we show that we can easily retrieve the optimal transport plan between the PCs from the solutions to these linear programs. Lastly, we study the empirical Wasserstein distance between a PC and a dataset, and show that we can estimate the PC parameters to minimize this distance through an efficient iterative algorithm.
\end{abstract}

\section{Introduction}

The Wasserstein distance is a statistical distance metric corresponding to the objective value taken by the optimal transport problem as proposed by Kantorovich's optimal transport framework that, given two probability measures, finds its optimal value at a coupling measure where the expected distance between the original two measures is minimized~\citep{kantorovichot}. Computing such a distance has proven extremely useful, with applications in generative modeling~\citep{arjovsky2017wgans}, data privacy~\citep{tcloseness}, and distributionally robust optimization~\citep{Rahimian_2022}. Providing a detailed mapping between probability measures, optimal transport maps have also proven useful for problems such as fairness auditing~\citep{Black_2020} and as generative models on their own~\citep{rout2022generative}.

\ac{Unfortunately, computing the Wasserstein distance efficiently and exactly is a difficult task for all but the simplest distributions. This has led to the popularity of Wasserstein-like distances that enable tractable computation, such as the Sliced Wasserstein distance \citep{bonneel:hal-00881872}, Tree Wasserstein distance \citep{treewasserstein} and the Mixture Wasserstein distance \cite{delon2020wasserstein}. Others address this challenge by finding approximate solutions, with common approaches being the Sinkhorn \cite{sinkhorn} and factored couplings \citep{factoredcouplings} optimal transport solvers.}

Modeling probability distributions in a way that enables tractable computation of queries is of great interest to the machine learning community. Probabilistic circuits (PCs)~\citep{probcirc} provide a unifying framework for representing many classes of tractable probabilistic models as computational graphs; within this framework, tractability of certain queries can be guaranteed by imposing structural properties on the computational graph of the circuit. This includes tractable marginal and conditional inference, as well as pairwise queries that compare two distributions such as Kullback-Leibler (KL) divergence and cross-entropy~\citep{liang2017CI,vergari2021atlas}. However, to the best of our knowledge, there is no existing algorithm that tractably computes the Wasserstein distance between two probabilistic circuits.

While algorithms for computing other statistical distance measures between PCs are well-established, the Wasserstein distance offers a distinct advantage in many applications. Measures such as KL divergence and cross-entropy are unbounded between distributions with disjoint supports; on the other hand, the $p$-Wasserstein distance is always bounded for distributions with finite $p$-th moments \citep[p. 107]{otbook}. Computing the Wasserstein distance also provides a bound for other statistical distance metrics such as the Prokhorov metric and the total-variation distance~\citep{bounding}. However, computing the Wasserstein distance is often intractable, even for models that are tractable for other probabilistic queries such as the KL divergence or cross-entropy, due to the inherent optimization problem required to be solved.

This paper focuses on computing (or bounding) the Wasserstein distance and optimal transport plan (i) between two PCs and (ii) between a PC and an empirical distribution. For (i) we propose a Wasserstein-type distance that upper-bounds the true Wasserstein distance and provide an efficient and exact algorithm to compute it between two circuits (Section~\ref{sec:otpc}). For (ii) we propose a parameter estimation algorithm for PCs that seeks to minimize the Wasserstein distance between a circuit and an empirical distribution (Section~\ref{sec:learning}). We empirically evaluate our proposed methods on randomly generated PCs as well as on a benchmark dataset (Section~\ref{sec:experiments}).

\section{Preliminaries}

We use capital letters ($X$) to denote random variables and lowercase letters ($x$) to denote their assignments. Boldface denotes a set of random variables and their assignments respectively (e.g., $\X$ and $\x$).

\paragraph{Wasserstein Distance and Optimal Transport}
    Let $P$ and $Q$ be probability measures on metric space $\mathbb{R}^n$. For $p\!\geq\!1$, the \emph{$p$-Wasserstein distance} between $P$ and $Q$ is defined~as:
    \begin{align}
        \wassDist_p^p(P,Q) = \inf_{\gamma \in \Gamma(P,Q)} \Ex_{\gamma(\x,\y)}[\norm{\x-\y}_p^p] \label{eq:wp-distance}
    \end{align}
    where \ac{$\norm{\cdot}_p$ denotes the vector $p$-norm and} $\Gamma(P,Q)$ denotes the set of all \emph{couplings}, i.e.\ joint distributions whose marginal distributions match $P$ and $Q$. That is, the following holds for all $\gamma \in \Gamma(P,Q)$:
    \begin{align}
        P(\x) = \int_{\mathbb{R}^n} \gamma(\x,\y) d\y, \quad Q(\y) = \int_{\mathbb{R}^n} \gamma(\x,\y) d\x. \label{eq:marginal-constraint}
    \end{align}
Here, the \emph{Wasserstein objective} of a (not necessarily optimal) coupling refers to the expectation inside the infimum in Equation~\ref{eq:wp-distance} taken over that coupling, and the \emph{Wasserstein distance} between two distributions refers to the value taken by the Wasserstein objective for the optimal coupling.
It can be shown that there is always a coupling that obtains the infimum above~\citep{otbook}. Such optimal coupling $\gamma^*$ induces a \textit{transport plan} $\x \mapsto \gamma^*(\x,.)$. If the coupling is deterministic, this is called a \textit{transport map} $\x \mapsto T(x)$ where $T(x)$ is the support of $\gamma^*(\x,.)$.

\paragraph{Probabilistic Circuits}
Many tractable probabilistic models---arithmetic circuits~\citep{darwiche2003differential}, sum-product networks~\citep{poon2011sum}, cutset networks~\citep{rahman2014cutset}, and more---can be understood through a unifying framework of \emph{probabilistic circuits}~\citep{probcirc}.
\begin{defn}\label{def:pcs}
    A probabilistic circuit (PC) $C$ over a set of 
    variables $\X$ is a parameterized, rooted directed acyclic graph (DAG) with three types of nodes: sum, product, and input nodes. 
    Each input node $n$ is associated with function $f_n$ that encodes a probability distribution over a variable $X_i \in \X$, also called its \emph{scope} $\scope(n)$. The set of child nodes for an internal node (sum or product) $n$ is denoted $\ch(n)$, and its scope is given by $\scope(n)=\bigcup_{c \in \ch(n)}\scope(c)$.
    Each sum node $n$ has normalized parameters $\theta_{n,c}$ for each child node $c$.
    For an assignment $\x$, let $\x_n$ denote its projection onto the scope $\scope(n)$ of node $n$. Then a PC rooted at node $n$ recursively defines a probability distribution $p_n(\x)$ as follows: 
    \begin{align*}
        p_n(\x) =\begin{cases}
        f_n(\x) & \text{if $n$ is an input node,} \\
      \prod_{c \in \ch(n)} p_c(\x_c) & \text{if $n$ is a product node,} \\
      \sum_{c \in \ch(n)}\theta_{n,c} p_c(\x_c) & \text{if $n$ is a sum node.}
    \end{cases}
    \end{align*}
\end{defn}
Probabilistic circuits admit exact and efficient computation of many probabilistic inference queries, enabled by enforcing certain structural constraints. 
In particular, throughout this paper we assume two properties, \textit{smoothness} and \textit{decomposability}, which enable tractable (polytime) computation of marginal and conditional queries. A PC is \emph{smooth} if every sum node $n\in C$ has the same scope as its children: $\forall n_i \in \ch(n)$, $\scope(n_i) = \scope(n)$; it is \emph{decomposable} if the children of every product node $n \in C$ have disjoint scopes: $\forall n_i \neq n_j \in \ch(n)$, $\scope(n_i) \bigcap \scope(n_j) = \emptyset$. Tractable computation of different query classes can be enabled by ensuring additional properties.\footnote{E.g., KL divergence between two \textit{compatible} and \textit{deterministic} PCs can be computed in polynomial time~\citep{vergari2021atlas}.}
Critically, enforcing these structural properties does not restrict the \textit{expressivity} of PCs (i.e., they can still represent any distribution), but may decrease their \textit{expressive efficiency} (i.e., an exponentially-sized circuit may be required when enforcing a constraint).

\section{Optimal Transport between Probabilistic Circuits}\label{sec:otpc}

We now consider the problem of computing Wasserstein distances and optimal transport plans between distributions represented by probabilistic circuits. For notational simplicity, suppose $P(\X)$ and $Q(\Y)$ are two PCs defining probability measures on a metric space, with a bijective mapping between variables in $\X$ and those in $\Y$; w.l.o.g., let $X_i$ and $Y_i$ map to each other. 
We also assume that the input distributions in the PCs are univariate,\footnote{We assume univariate input nodes for notational simplicity, but this is not a requirement.} and also allow constant-time computation of the Wasserstein distance---following the standard assumption of tractability of input distributions for tractable inference on PCs. In particular, this is the case for $p$-Wasserstein distance between categorical distributions and for the 2-Wasserstein distance between Gaussian distributions. Unfortunately, even with these assumptions, computing the Wasserstein distance between probabilistic circuits is computationally hard. 
\begin{thm}\label{thm:w-hardness}
    Suppose $P$ and $Q$ are probabilistic circuits over $n$ Boolean variables. Then computing the $\infty$-Wasserstein distance between $P$ and $Q$ is coNP-hard.
\end{thm}
In fact, the above is true even when the PCs satisfy stronger structural constraints (determinism and structured decomposability) that enable tractable inference of hard queries such as maximum-a-posteriori (MAP)~\cite{choi2017relaxing} and entropy~\citep{vergari2021atlas} and even closed-form maximum-likelihood parameter estimation.
At a high level, the proof proceeds by reducing from the problem of deciding consistency of two OBDDs (a type of deterministic and structured-decomposable circuits) which is NP-hard~\citep[Lemma 8.14]{meinel1998algorithms}. In particular, given the two OBDDs, we can construct two deterministic and structured-decomposable PCs in polynomial time such that the input OBDDs are consistent iff $\wassDist_\infty$ between the PCs is not 1. We refer to Appendix~\ref{proof:hardness} for a detailed proof.

\begin{figure}
    \centering
    \scalebox{0.4}{
    \begin{tabular}{c}
    \begin{tikzpicture}
      \contnodeleft[line width=\midlinewidth, draw=tomato2]{n1}{$\textbf{X}_1$}
      \contnodeleft[line width=\midlinewidth, below=\middist of n1, draw=lacamdarklilac4]{n2}{$\textbf{X}_2$}
      \wprodnode[line width=\midlinewidth, right=\middist of n2, draw=petroil2]{p1}{petroil2}
      \edge[line width=\midlinewidth]{p1}{n1,n2};
      
      \contnodeleft[line width=\midlinewidth, below=\middist of n2, draw=lacamdarklilac4]{n3}{$\textbf{X}_2$}
      \sumnode[line width=\midlinewidth, right=\middist of n3, draw=lacamdarklilac4]{s1}
      \contnodeleft[line width=\midlinewidth, below=\middist of s1, draw=tomato2]{n4}{$\textbf{X}_1$}
      
      \edge[line width=\midlinewidth]{s1}{n2,n3};
      
      \wprodnode[line width=\midlinewidth, right=\middist of s1, draw=petroil2]{p2}{petroil2}
      \edge[line width=\midlinewidth]{p2}{s1,n4};
      
      \sumnode[line width=\midlinewidth, right=\middist * 1.75 of p1, draw=petroil2]{s2}
      \edge[line width=\midlinewidth]{s2}{p1,p2};
      
      \contnodeleft[line width=\midlinewidth, above=\middist of s2, draw=pink3]{n5}{$\textbf{X}_3$}
      \wprodnode[line width=\midlinewidth, right=\middist of s2, draw=gold4]{p3}{gold4}
      \edge[line width=\midlinewidth]{p3}{s2,n5};
    
      \contnodeleft[line width=\midlinewidth, right = \middist * 12 of n1, draw=lacamdarklilac4]{n6}{$\textbf{Y}_2$}
      \contnodeleft[line width=\midlinewidth, below = \middist of n6, draw=tomato2]{n7}{$\textbf{Y}_1$}
      \wprodnode[line width=\midlinewidth, right=\middist of n7, draw=petroil2]{p4}{petroil2}
      \sumnode[line width=\midlinewidth, right=\middist of p4, draw=petroil2]{s3}
      \wprodnode[line width=\midlinewidth, right=\middist of s3, draw=gold4]{p5}{gold4}
      \edge[line width=\midlinewidth]{p4}{n6,n7};

      \wprodnode[line width=\midlinewidth, below=\middist of p4, draw=petroil2]{p8}{petroil2}
      \contnodeleft[line width=\midlinewidth, above = \middist of s3, draw=pink3]{n81}{$\textbf{Y}_3$}
      \edge[line width=\midlinewidth]{s3}{p8,p4};
      
      \wprodnode[line width=\midlinewidth, below=\middist of p8, draw=petroil2]{p6}{petroil2}
      \sumnode[line width=\midlinewidth, right=\middist of p6, draw=petroil2]{s4}
      \contnodeleft[line width=\midlinewidth, below = \middist of s4, draw=pink3]{n8}{$\textbf{Y}_3$}
      \edge[line width=\midlinewidth]{p5}{s3,n81};
      \contnodeleft[line width=\midlinewidth, left = \middist of p6, draw=lacamdarklilac4]{n9}{$\textbf{Y}_2$}
      \wprodnode[line width=\midlinewidth, right=\middist of s4, draw=gold4]{p7}{gold4}
      \contnodeleft[line width=\midlinewidth, below = \middist of n9, draw=tomato2]{n10}{$\textbf{Y}_1$}
      \edge[line width=\midlinewidth]{p8}{n9,n7};
      \edge[line width=\midlinewidth]{p6}{n9,n10};
      \edge[line width=\midlinewidth]{s4}{p6,p8};
      \edge[line width=\midlinewidth]{p7}{s4,n8};
      
      \sumnode[line width=\midlinewidth, above=\middist of p7, xshift=20px, draw=gold4]{s5}
      \edge[line width=\midlinewidth]{s5}{p7,p5};
      
      \bernodeleft[line width=\midlinewidth, right = \middist * 8 of n81, draw=gold4]{123}
      
      \bernodeleft[line width=\midlinewidth, below = \middist * 2 of 123, xshift=30px, draw=petroil2]{23}
      
      \bernodeleft[line width=\midlinewidth, left = \middist * 2 of 23, draw=pink3]{1}{3}
      
      \bernodeleft[line width=\midlinewidth, below = \middist * 2 of 23, xshift=30px, draw=lacamdarklilac4]{3}{2}
      
      \bernodeleft[line width=\midlinewidth, left = \middist * 2 of 3, draw=tomato2]{2}{1}
      \edge[line width=\midlinewidth]{123}{1,23};
      \edge[line width=\midlinewidth]{23}{2,3};
    \end{tikzpicture} \\
    
     \end{tabular}
    }
    \caption{\small Compatible circuits over $\mathbf{X}\!=\!\{X_1,X_2,X_3\}$ and $\mathbf{Y}\!=\!\{Y_1,Y_2,Y_3\}$. Nodes in the same color have same scope, \ac{and the scope decomposition is visualized on the right}.}
    \label{fig:compatibility}
\end{figure}

Theorem \ref{thm:w-hardness} shows that computing the $\infty$-Wasserstein distance between two PCs is computationally hard. Whether computing $\wassDist_p$ for some other fixed $p$ (such as $p=1$ or $2$) is NP-hard is still an open question---there only exist efficient algorithms that bound this quantity between GMMs (which can be viewed as a special case of PCs), rather than compute it exactly \citep{delon2020wasserstein, chen2018optimal}. However, we are interested in efficiently computing or upper-bounding $\wassDist_p$ for \textit{arbitrary} $p$, including $\wassDist_\infty$. Thus, to address this computational challenge, we consider a Wasserstein-type distance between PCs by restricting the set of coupling measures to be PCs of a particular structure. Furthermore, we derive the structural conditions on the input PCs required to construct such structure and find the parameters that minimize the Wasserstein objective in time quadratic in the size of the input circuits.

\subsection{\texorpdfstring{$\circDist_p$}{CWp}: A Distance between Circuits}

\ac{The building blocks of PCs are sum, product, and input nodes. Sum nodes encode mixtures of their child distributions; algorithms to compute optimal transport distances between mixtures of distributions are well-studied~\citep{delon2020wasserstein, chen2018optimal}, and formulate these distances as a solution to the weighted discrete optimal transport problem between mixture components in each distribution. 
Next, product nodes encode factorizations of distributions; since the $p$-norm $\norm{\cdot}_p^p$ is \emph{separable} across dimensions, $\wassDist_p$ between two factorized distributions is the sum of the Wasserstein distances between the corresponding factors.
Lastly, input nodes provide a natural base-case for an optimal transport framework: we assume that we can compute the Wasserstein distance and corresponding transport plan between two input nodes in constant-time, which is the usual assumption for tractable inference using PCs.
Given this, we ask a natural question: can optimal transport distances and transport plans between PCs be computed recursively?}

\ac{We show that the answer is yes. To do this,} we propose the notion of \emph{coupling circuits} between two \emph{compatible} 
PCs, and introduce a Wasserstein-type distance $\circDist_p$ which restricts the coupling set in Equation~\ref{eq:wp-distance} to be circuits of this form. We then exploit the structural properties guaranteed by coupling circuits to derive efficient algorithms for computing $\circDist_p$ and associated transport plan.

\begin{figure}
    \centering
    \scalebox{0.6}{
    \begin{tabular}{c}
    \begin{tikzpicture}
      \bernode[line width=\midlinewidth]{n1}{$P_1(\textbf{X}_1)$}
      \bernode[line width=\midlinewidth, right=\middist * 1.25 of n1]{n2}{$P_2(\textbf{X}_2)$}
      \prodnode[line width=\midlinewidth, above=\halfdist of n1, xshift=20pt,label={\small $P(\textbf{X})$}]{p1}
      \edge[line width=\midlinewidth]{p1}{n1,n2};

      \bernode[line width=\midlinewidth, right = \middist * 2 of n2]{n3}{$Q_1(\textbf{Y}_1)$}
      \bernode[line width=\midlinewidth, right=\middist * 1.25 of n3]{n4}{$Q_2(\textbf{Y}_2)$}
      \prodnode[line width=\midlinewidth, above=\halfdist of n3, xshift=20pt,label={\small $Q(\textbf{Y})$}]{p2}
      \edge[line width=\midlinewidth]{p2}{n3,n4};
    
      \draw[->, line width=\midlinewidth](5,1) -- (7,1) node[midway, above] {\textsc{Cpl}$(P,Q)$};
    
      \bernode[line width=\midlinewidth, right = \middist * 3 of n4]{n5}{\small $\textsc{Cpl}(P_1, Q_1)$}
      \bernode[line width=\midlinewidth, right=\middist * 3.25 of n5]{n6}{\small $\textsc{Cpl}(P_2, Q_2)$}
      \prodnode[line width=\midlinewidth, above=\halfdist of n5, xshift=37pt,label={\small $C(\textbf{X,Y})$}]{p3}
      \edge[line width=\midlinewidth]{p3}{n5,n6};
    \end{tikzpicture} \\
    \begin{tikzpicture}
      \sumnode[line width=\midlinewidth, label={\small $P(\textbf{X})$}]{p1}
      \bernode[line width=\midlinewidth, below=\halfdist of p1, xshift=-\middist*0.7]{n1}{$P_1$}
      \bernode[line width=\midlinewidth, below=\halfdist of p1, xshift=\middist*0.7]{n2}{$P_2$}
      \weigedge[line width=\midlinewidth, left=0.1]{p1}{n1}{$\theta_1^P$};
      \weigedge[line width=\midlinewidth, right=0.1]{p1}{n2}{$\theta_2^P$};
    
      \sumnode[line width=\midlinewidth, right=\middist*3 of p1,label={\small $Q(\textbf{Y})$}]{p2}
      \bernode[line width=\midlinewidth, below=\halfdist of p2, xshift=-\middist*1.5]{n3}{$Q_1$}
      \bernode[line width=\midlinewidth, below=\halfdist of p2]{n4}{$Q_2$}
      \bernode[line width=\midlinewidth, below=\halfdist of p2, xshift=\middist*1.5]{n5}{$Q_3$}
      \weigedge[line width=\midlinewidth, left=0.1]{p2}{n3}{\small $\theta_1^Q$};
      \weigedge[line width=\midlinewidth, right=-0.1]{p2}{n4}{\small $\theta_2^Q$};
      \weigedge[line width=\midlinewidth, right=0.1]{p2}{n5}{\small $\theta_3^Q$};
    
      \draw[->, line width=\midlinewidth](3.5,0.5) -- (5.5,0.5) node[midway, above] {\textsc{Cpl}$(P,Q)$};

      \sumnode[line width=\midlinewidth, right=\middist*8.5 of p2, yshift=20pt,label={\small $C(\textbf{X,Y})$}]{p3}
    
      \bernode[line width=\midlinewidth, below=\halfdist of p3, xshift=-\middist*6,yshift=-20pt]{n6}{\tiny $\textsc{Cpl}(P_1,Q_1)$}
      \bernode[line width=\midlinewidth, below=\halfdist of p3, xshift=-\middist*3.5,yshift=-20pt]{n7}{\tiny $\textsc{Cpl}(P_1,Q_2)$}
      \bernode[line width=\midlinewidth, below=\halfdist of p3, xshift=-\middist*1.1,yshift=-20pt]{n8}{\tiny $\textsc{Cpl}(P_1,Q_3)$}
      \bernode[line width=\midlinewidth, below=\halfdist of p3, xshift=\middist*1.1,yshift=-20pt]{n9}{\tiny $\textsc{Cpl}(P_2,Q_1)$}
      \bernode[line width=\midlinewidth, below=\halfdist of p3, xshift=\middist*3.5,yshift=-20pt]{n10}{\tiny $\textsc{Cpl}(P_2,Q_2)$}
      \bernode[line width=\midlinewidth, below=\halfdist of p3, xshift=\middist*6,yshift=-20pt]{n11}{\tiny $\textsc{Cpl}(P_2,Q_3)$}
      \weigedge[line width=\midlinewidth, above=0.5pt, rotate=30]{p3}{n6}{\small $\theta_{1,1}$};
      \weigedge[line width=\midlinewidth, left=8pt, rotate=43]{p3}{n7}{\small $\theta_{1,2}$};
      \weigedge[line width=\midlinewidth, left=0.2pt, yshift=-5pt]{p3}{n8}{\small $\theta_{1,3}$};
      \weigedge[line width=\midlinewidth, right=0.2pt, yshift=-5pt]{p3}{n9}{\small $\theta_{2,1}$};
      \weigedge[line width=\midlinewidth, right=8pt, rotate=-43]{p3}{n10}{\small $\theta_{2,2}$};
      \weigedge[line width=\midlinewidth, above=0.5pt, rotate=-30]{p3}{n11}{\small $\theta_{2,3}$};
    \end{tikzpicture}\\

    \large $\theta_{1,1}+\theta_{1,2}+\theta_{1,3}=\theta_1^P$ {} {} {} {} {} {} {}
   \large $\theta_{2,1}+\theta_{2,2}+\theta_{2,3}=\theta_2^P$  \\\\
    \large $\theta_{1,1}+\theta_{2,1}=\theta_1^Q$ {} {} {} {} {} {} {} \large $\theta_{1,2}+\theta_{2,2}=\theta_2^Q$ {} {} {} {} {} {} {}
    \large $\theta_{1,3}+\theta_{2,3}=\theta_3^Q$
    
     \end{tabular}
    }
    \caption{\small Recursive construction of coupling circuits. (Top) Product nodes couple children with corresponding scopes. (Bottom) Sum nodes couple the Cartesian product of children, with marginal constraints for the parameters.}
    \label{fig:couplingcircuits}
\end{figure}

\begin{defn}[Circuit compatibility~\citep{vergari2021atlas}]\label{def:compatibility}    
    Two smooth and decomposable PCs $P$ and $Q$ over RVs $\X$ and $\Y$, respectively, are \emph{compatible} if the following two conditions hold: (i) there is a bijective mapping $\leftrightarrow$ between RVs $X_i$ and $Y_i$, and (ii) any pair of product nodes $n \in P$ and $m \in Q$ with the same scope up to the bijective mapping are mutually compatible and decompose the scope the same way---that is, if $n$ and $m$ have scopes $\X$ and $\Y$ and $\X \leftrightarrow \Y$, then $n$ and $m$ have the same number of children, and for each child of $n$ with scope $\X_i$ there is a corresponding child of $m$ with scope $\Y_i$ such that $\X_i \leftrightarrow \Y_i$. Such pair of nodes are called \emph{corresponding} nodes.
\end{defn}
That is, two circuits are compatible if they have the same \textit{hierarchical scope partitioning}; see Figure~\ref{fig:compatibility} for an example pair of compatible circuits. This does not require the circuit structures to be the same.

\begin{defn}[Coupling circuit]\label{def:couplingcircuits}
    A \emph{coupling circuit} $C$ between two compatible PCs $P$ and $Q$ with scopes $\X$ and $\Y$, respectively, is a PC with the following properties. (i) Each node $r \in C$ is recursively a coupling of a pair of nodes $n \in P$ and $m \in Q$.\footnote{The coupling circuit has the same structure (but not parameters) as the product circuit~\citep{vergari2021atlas} of $P$ and $Q$. Informally, we construct a \textit{cross product} of children at each pair of sum nodes, and the product of corresponding children at each pair of product~nodes.} (ii) Each node $r \in C$ that is a coupling of sum nodes $n\in P, m\in Q$ with edge weights $\{\theta_i\}$ and $\{\theta_j\}$ has edge weights $\{\theta_{i,j}\}$ such that $\sum_i\theta_{i,j}=\theta_j$ and $\sum_j\theta_{i,j}=\theta_i$ for all $i$ and $j$.
\end{defn}

\ac{Note that the above definition does not assign parameters to a coupling circuit structure; rather, the properties of a coupling circuit restrict the set of feasible parameters.}
The second property described above ensures that coupling circuit $C$ satisfies the marginal-matching constraints in Equation~\ref{eq:marginal-constraint} with respect to $P$ and $Q$ (proof in Appendix~\ref{proof:marginalmatching}) by requiring that the sub-circuit rooted at any internal node in the coupling circuit matches marginals to the corresponding nodes in the original circuits (which is a stronger constraint than the entire coupling circuit simply matching marginal distributions to the original circuits).
We are now ready to define our proposed distance metric between PCs, which is the minimum Wasserstein objective obtained by a valid parameterization of their coupling circuit.
\begin{defn}[Circuit Wasserstein distance]\label{def:cw}
    Let $P(\X)$ and $Q(\Y)$ be compatible PCs and $C_{\theta}(\X,\Y)$ \ac{a} coupling circuit parameterized by $\theta$. The \emph{$p$-Circuit Wasserstein distance} between $P$ and $Q$ is defined as:
    \begin{equation*}
        \emph{\circDist}_p^p(P,Q) = \inf_{\theta} \Ex_{C_{\theta}(\x,\y)}[\norm{\x-\y}_p^p].
    \end{equation*}
\end{defn}
We now investigate some properties of $\circDist_p$. First, we note that \ac{$\circDist_p$ upper-bounds the true Wasserstein distance between PCs as both are infima of the same Wasserstein objective, while the feasible set of couplings for $\circDist_p$ is more~restrictive.} Moreover, $\circDist_p$ is a metric on any set of compatible circuits, which is contrary to some other statistical measures such as KL divergence used to compare distributions. See Appendix~\ref{proof:metric} for a formal proof.
\begin{prop}\label{prop:metric}
    For any set $\mathcal{C}$ of compatible circuits, $\circDist_p$ defines a metric on $\mathcal{C}$.
\end{prop}
Lastly, \ac{the optimal} coupling circuit $C(\X,\Y)$ corresponding to $\circDist_p(P,Q)$ induces a (albeit not necessarily optimal) transport plan that maps a point $\x$ to a distribution $C(\Y |\X=\x)$ and vice versa.

\subsection{Tractable Computation of \texorpdfstring{$\circDist_p$}{CWp}}

\ac{The structure of coupling circuits enables efficient computation of both the Circuit Wasserstein distance and optimal coupling circuit parameters. Informally, the $\circDist_p$ optimization problem at a coupling circuit's sum node is simply the weighted discrete optimal transport problem, where weights are the original sum node parameters and distances are the $\circDist_p$ distances between child circuits. The $\circDist_p$ optimization problem also decomposes at coupled product nodes, as product nodes encode factored distributions. Thus, we leverage the desirable properties of coupling circuits to exactly and efficiently computes the Circuit Wasserstein distance of two compatible PCs, as well as the optimal coupling circuit parameters. The key observation enabling our algorithm is that the Wasserstein objective for a given parameterization of the coupling circuit can be computed recursively through a single feedforward pass through the circuit, and can also be minimized over its parameters in a single forward pass.}

\paragraph{Recursive Computation of the Wasserstein Objective}
Let $C(\X,\Y)$ be a parameterized coupling circuit and $g(n) = \Ex_{C_n}[\norm{\x-\y}_p^p]$ the corresponding $\circDist_p$-objective function at each node $n \in C$. We can write $g(n)$ recursively as follows (see Appendix~\ref{proof:wcomp} for correctness proof):
\begin{equation}\label{eq:cw-obj}
   g(n) \!=\!\!\begin{cases} 
      \wassDist_p^p(c_1,c_2) \!\!&\!\! \text{if $n$ is $\otimes$ of input nodes,} \\
      \sum_{c\in\ch(n)} g(c) \!\!&\!\! \text{if $n$ is $\otimes$ of internal nodes,} \\
      \sum_{c\in\ch(n)}\! \theta_{n,c} g(c) \!\!&\!\! \text{if $n$ is $\oplus$ node.} 
    \end{cases}
\end{equation}

\vspace{-10pt}
Thus, we can push computation of the Wasserstein objective down to the leaf nodes of a coupling circuit, and our algorithm only requires a closed-form solution for $\wassDist_p$ between univariate input distributions as the base case.

\begin{algorithm}[t]
   \caption{$\textsc{Couple}(n,m)$: {\small coupling circuit that optimizes $\circDist_p^p(n,m)$ of compatible PCs rooted at nodes $n,m$}}\label{algimp:couple-cw}
   
\textbf{Note:} We omit calls to a cache storing previously-computed coupling circuits $\textsc{Couple}(n,m)$ for simplicity.

\begin{algorithmic}[1]
   \If{$n,m$ are input nodes} 
   \State $r \gets \text{input node denoting OT plan between $n,m$}$

    \ElsIf{$n,m$ are sum nodes}
    \State $r \gets $ new sum node with parameters $\theta_{i,j}$
    \ForAll{$c_i \in n$.children, $c_j \in m$.children}
        \State $r$.children[$i, j$] $ \gets\textsc{Couple}(c_i,c_j$) 
    \EndFor
    
    \State $\lp \gets \begin{cases} 
    \begin{array}{ll@{}ll}
    \text{min}  & \sum_{i,j} \circDist_p(r\text{.children}[i,j])*\theta_{i,j} &\\
    \text{s.t.}& \forall i,\, \sum_j\theta_{i,j}=\theta_i \\&\forall j,\, \sum_i\theta_{i,j}=\theta_j \\ & \theta_{i,j} \in [0,1]
    \end{array}\end{cases} $
    
    \State solve $\lp$ \Comment{Solve for optimal parameters $\theta_{i,j}$}
    
    \ElsIf{$n,m$ are product nodes}
    \State $r \gets $ new product node
    \ForAll{$c_1 \in n$.children, $c_2 \in m$.children \textbf{where} $\scope(c_1)=\scope(c_2)$}
        \State add \textsc{Couple}($c_1,c_2$) to $r$.children 
    \EndFor
    \EndIf
    
    \State \textbf{return} $r$
\end{algorithmic}
\end{algorithm}

\paragraph{Recursive Computation of the Optimal Coupling Circuit Parameters for $\circDist_p$}
Leveraging the recursive properties of the Wasserstein objective, we can compute the optimal parameters of the coupling circuit by solving a small linear program at each sum node. Algorithm~\ref{algimp:couple-cw} details the construction of a coupling circuit (illustration in Figure~\ref{fig:couplingcircuits}) and finding the optimal parameters to compute $\circDist_p$.

Specifically, we wish to find $\min_{\theta} g(n)$ where $g(n)$ can be written recursively as in Equation~\ref{eq:cw-obj}. By this definition, at sum nodes we can minimize the Wasserstein objective at each child independently then solve a linear program using the objective value at the child nodes as constants: given the optimal $g(c)$ for each child node $c$ of $n$, we can rewrite $\min_\theta g(n)=\min_\theta \sum_{c \in \ch(n)} \theta_c g(c)$ to see that solving for the sum node parameters reduces to solving a linear program. We can decompose the optimization problem this way because the optimization at children are independent of the parent parameters.
At a product node, we can again push the problem down to the children: $\min_{\theta} \sum_{c\in\ch(n)} g(c) = \sum_{c\in\ch(n)} (\min_\theta g(c))$, because the children nodes $c\in\ch(n)$ have disjoint scopes due to decomposability and thus do not share any parameters.

Since the time to solve the linear program at each sum node depends only on the number of children of the sum node, which is bounded, we consider this time constant when calculating the runtime of the full algorithm. Thus, we can compute $\circDist_p$ and the corresponding transport plan between two circuits in time linear in the size of the coupling circuit, or equivalently, quadratic in the size of the original input circuits.
Appendix~\ref{proof:exactcw} presents correctness proof of the algorithm in more detail.

Compatibility has been shown to enable tractable algorithms for many pairwise queries on PCs, including information-theoretic divergences such as the Kullback-Leibler, Jensen-Shannon, and R\'enyi divergences~\citep{vergari2021atlas}. Interestingly, while such divergences are commonly considered to be easier to compute than the Wasserstein distance (as they do not require solving an optimization problem), $\circDist_p$ computation is actually \emph{easier}, as algorithms to compute the former all require an additional structural property called \emph{determinism}. Furthermore, while our algorithm for $\circDist_p$ requires the two circuits to be compatible, any pair of arbitrary non-compatible circuits may be transformed into two structured-decomposable circuits and then made compatible---albeit with a worst-case exponential increase in circuit size \citep{decolnet2021compilationsuccinctnessresultsarithmetic, zhang2024restructuringtractableprobabilisticcircuits}. Lastly, current state-of-the-art PC learning algorithms naturally allow us to learn compatible circuit structures---assuming we assign the bijective mapping ourselves \citep{strudel, hclt}.

\subsection{Relation to Distance between GMMs} \label{sec:otgmms}

As probabilistic circuits with Gaussian input distributions can be interpreted as deep, compact representations of Gaussian mixture models (GMMs), existing works studying optimal transport for GMMs~\citep{chen2018optimal,delon2020wasserstein} are highly relevant.
In particular, our proposed notion of Circuit Wasserstein distance is closely related to the Mixture Wasserstein distance ($\gmmDist_2$) introduced by \citet{delon2020wasserstein}, who also derived an upper bound on the true Wasserstein distance by restricting the coupling set to a GMM structure with quadratic number of components and computed this metric by solving a linear program.

We can directly leverage this algorithm to compute a bound on $p$-Wasserstein distance between PCs. Specifically, we can ``unroll'' PCs with Gaussian inputs into their shallow representations which correspond to GMMs and them compute $\gmmDist_p$ between those unrolled GMMs. 
However, the shallow representation of a PC may be exponentially larger than the size of the original circuit, making this naive approach intractable; nevertheless, we consider this approach as a baseline for our proposed approach and provide a detailed runtime comparison in Section~\ref{sec:experiments}.
Furthermore, we observe that $\gmmDist_p$ will be no larger than our proposed $\circDist_p$ because a coupling circuit can also be unrolled into a GMM and thus must be in the coupling set for $\gmmDist_p$.

\section{PC Parameter Learning using Wasserstein Distance}
\label{sec:learning}

Motivated by past works that train generative models by minimizing the Wasserstein distance between the model and the empirical data distribution~\citep{routgenerative,salimans2018improving,tolstikhin2018wasserstein,arjovsky2017wgans}, we investigate the applicability of minimizing the Wasserstein distance between a PC and data as a means of learning the parameters of a given PC structure.

Formally, suppose we have a dataset $\data=\{\y^{(k)}\}_{k=1}^n$ that induces the empirical probability measure $\hat{Q}$. Then for a given PC structure, we find its parameters $\theta$ to optimize the following:
\begin{align} \min_{\theta} \wassDist_p^p(P_\theta, \hat{Q}) = \min_{\theta} \inf_{\gamma \in \Gamma(P_\theta,\hat{Q})} \Ex_{\gamma(\x,\y)}[\norm{\x-\y}_p^p] \nonumber\\ = \min_{\theta} \inf_{\gamma \in \Gamma(P_\theta,\hat{Q})} \frac{1}{n} \sum_{k=1}^n \Ex_{\gamma(\x|\y^{(k)})}\left[\norm{\x-\y^{(k)}}_p^p\right]. \label{eq:emp-wp-dist} \end{align}

Note that 
Equation~\ref{eq:emp-wp-dist} comes from rewriting $\gamma(\x,\y)=\gamma(\x|\y)\gamma(\y)$, then applying linearity of expectation since $\hat{Q}$ is an empirical distribution.
Unfortunately, solving the above optimization problem is computationally hard.

\begin{thm}\label{thm:mwhard}
    Suppose $P_{\theta}$ is a smooth and decomposable probabilistic circuit, and $\hat{Q}$ is an empirical distribution induced by a dataset $\data=\{\y^{(k)}\}_{k=1}^n$. Then computing the parameters $\theta$ that minimizes the empirical Wasserstein distance $\wassDist_p^p(P_\theta, \hat{Q})$ (i.e., solving Equation~\ref{eq:emp-wp-dist}) is NP-hard.
\end{thm}
We can show the above by a reduction from $k$-means clustering (Appendix~\ref{proof:mwhard}).

\subsection{Wasserstein-Minimization: An Iterative Algorithm}

We again tackle this computational hardness by imposing a circuit structure on the coupling measure, allowing us compute the Wasserstein objective and optimize it efficiently.

\begin{defn}[Empirical Circuit Wasserstein distance]\label{def:ecw}
    Let $P$ be a PC distribution and $\hat{Q}$ an empirical distribution induced by a dataset $\data=\{\y^{(k)}\}_{k=1}^n$. The $p$-Empirical Circuit Wasserstein distance between $P$ and $\hat{Q}$ is
    \begin{equation*}
        \emph{\empCircDist}_p^p(P,\hat{Q}) = \min_{\gamma} \frac{1}{n} \sum_{k=1}^n \Ex_{\gamma(\x|\y^{(k)})}\left[\norm{\x-\y^{(k)}}_p^p\right],
    \end{equation*}
    where $\gamma(\x,\y^{(k)})$ satisfies the following: (i) for each $k\in \{1,\dots,n\}$, $\gamma(.,\y^{(k)})$ is a PC with the same structure as $P$ (but not necessarily the same parameters) that normalizes to $1/n$, and (ii) $\sum_{k=1}^n \gamma(\x,\y^{(k)}) = P(\x)$.
\end{defn}
A coupling satisfying the above structure clearly satisfies the marginal constraints and is in $\Gamma(P,\hat{Q})$. Therefore, the empirical Circuit Wasserstein distance upper-bounds the empirical Wasserstein distance: $\wassDist_p(P, \hat{Q}) \leq \empCircDist_p(P,\hat{Q})$. We will thus learn the parameters of PCs by minimizing this upper bound, which can be computed efficiently as follows.

We now present our \emph{iterative algorithm} for minimum-Wasserstein parameter learning. 
In particular, we wish to learn the circuit parameters $\theta$ that minimizes $\empCircDist_p^p(P_\theta,\hat{Q})$ which is in turn a minimization problem over couplings $\gamma$. Thus, we alternate between (i) optimizing the coupling given the current circuit parameters and (ii) updating the circuit parameters given the current coupling.

Let us first discuss step (i) which computes $\empCircDist_p^p(P_\theta,\hat{Q})$ for a given $\theta$ and in the process finds the corresponding coupling $\gamma$. First, rather than materializing a PC to represent $\gamma(.,\y^{(k)})$ for each $k$ as described in Definition~\ref{def:ecw}, we can equivalently model a single coupling circuit $\gamma$ as having the same structure as $P$ and a set of parameters $\{w_{r,c,k}\}_{k=1}^n$ for each parameter $\theta_{r,c}$ in $P$.
Then optimizing the coupling circuit parameters amounts to minimizing the expected distance according to the coupling distribution, similar to computing $\circDist$, and can be done efficiently by solving a small linear program at each sum node. Here, we have the following marginal-matching constraints: $\sum_{k} w_{r,c,k} \!=\! \theta_{r,c}$ for each sum node $r$ and child $c$ and $\sum_c w_{r,c,k}\!=\! 1/n$ for each~$k$.

Interestingly, the above linear program at each sum node is a variation of the continuous knapsack problem \citep{ctsknapsack} and thus has a closed-form solution. In particular, the solution results in a coupling circuit with each weight $w_{c,k}$ being either $\frac{1}{n}$ or zero (details in Appendix~\ref{pf:closedform}). Intuitively, the coupling circuit parameters $w$ describe how each data point is routed through the circuit; because the optimal coupling is deterministic---each data point is either routed wholly through an edge or not at all---we obtain a \textit{transport map} between the learned PC and empirical distribution.

Next, we discuss step (ii) which estimates the parameters $\theta$ of PC $P$ from a given coupling $\gamma$. Because the coupling has the same structure as $P$, and its weights $\{w_{r,c,k}\}$ satisfy marginal-matching constraints, we can simply extract the PC parameters: $\theta_{r,c} = \sum_{k=1}^n w_{r,c,k}$.

The above two steps are repeated iteratively until convergence; a pseudocode for the complete algorithm is provided in Appendix~\ref{alg:minwass}). 
Due to the closed-form solution of the LP, the time complexity of one iteration of our algorithm is linear in both the size of the circuit and the size of the dataset, and our algorithm is also guaranteed to converge (potentially to a local minimum) as the empirical Wasserstein objective is non-increasing in every iteration 
(Appendix~\ref{pf:monoton}). Nevertheless, finding the global optimum parameters minimizing the Wasserstein distance is still NP-hard, and our proposed efficient algorithm may get stuck at a local minimum, similar to existing maximum-likelihood parameter learning approaches.

In an effort to avoid getting stuck at a poor local minimum, we also introduce a variant of Wasserstein-Minimization called Stochastic Wasserstein-Minimization. Simply, instead of routing each data point optimally every time, we route each data point optimally with probability $1-p$ and randomly with probability $p$.

We briefly remark on interesting parallels between our proposed Wasserstein-Minimization (WM) methods and Expectation-Maximization (EM) for maximum-likelihood parameter learning. EM is an iterative algorithm that alternates between (i) computing the expected likelihood (marginalizing out the latent variables) of current parameters in the E-step and (ii) estimating the parameters that maximize this in the M-step, which is analogous to the two steps of WM: (i) computing the ECW for current parameters and (ii) updating the parameters to minimize ECW.

\vspace{-3pt}
\section{Experiments}
\label{sec:experiments}

In this section, we empirically evaluate our proposed algorithm for computing $\circDist_p$\footnote{Implementation and experiments can be found here: \url{https://github.com/aciotinga/pc-optimal-transport}} against the algorithm for computing $\gmmDist_p$ proposed by \citet{delon2020wasserstein}, the exact computation of $\wassDist_p$ via solving a linear optimization problem, \ac{and an approximation of $\wassDist_p$ via the Sinkhorn algorithm~\citep{sinkhorn}---an approximate optimal transport solver between point clouds}. We then compare the modeling capabilities of circuits learned using our proposed Wasserstein-Minimization (WM) algorithm against the Expectation-Maximization (EM) algorithm for PCs. Specifically, we aim to answer the following questions:
\vspace{-5pt}
\begin{enumerate}
    \item \label{q1} How does the runtime of our algorithm for $\circDist_p$ scale with the size of the circuit in practice, and how does that compare to $\gmmDist_p$ and $\wassDist_p$ computation?
    \item \label{q2} How effective is $\circDist_p$ as a proxy metric for $\wassDist_p$?
    \item How useful is the transport plan given by a coupling circuit?
    \item How do circuits learned using our WM algorithm compare with those learned using EM?
\end{enumerate}

\begin{figure}
    \centering
    \includegraphics[width=0.47\linewidth]{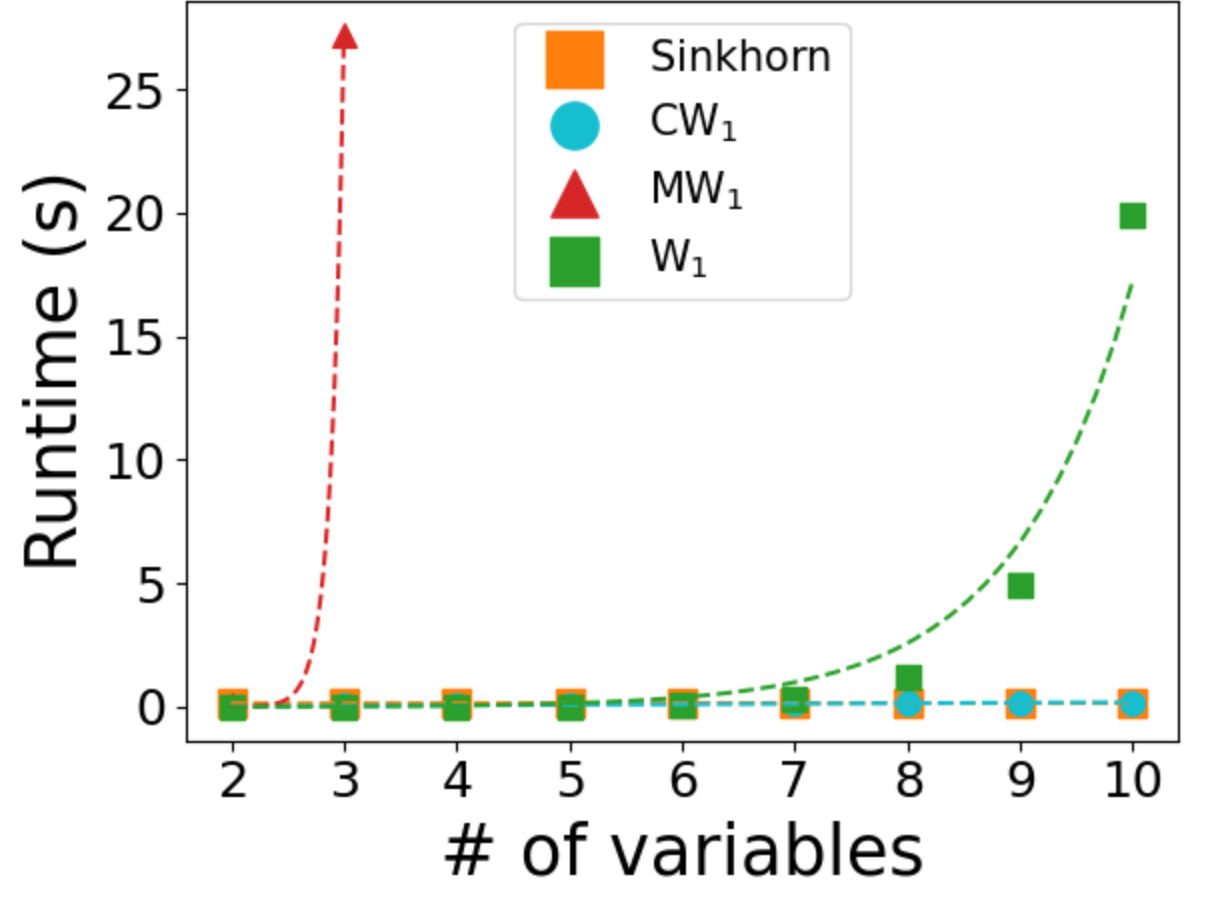}
    \hspace{5pt}
    \includegraphics[width=0.47\linewidth]{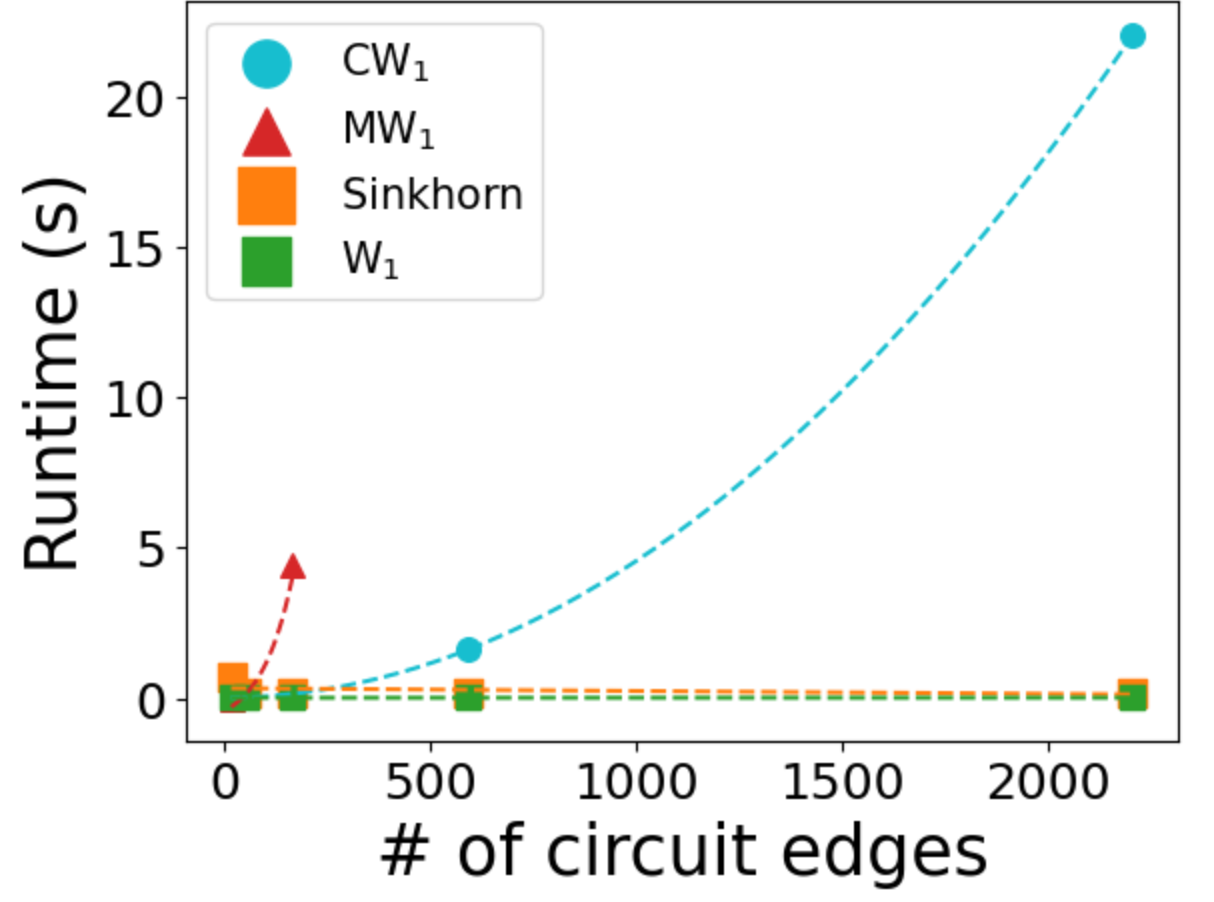}
    \caption{\small Runtime of Wasserstein-type distance computation using our approach (blue dots) and the baselines ($\gmmDist_1$ red triangles, $\wassDist_1$ green squares, and Sinkhorn distance with orange squares). \textit{Left:} Fixed $k=4$, variable $v$. \textit{Right:} Fixed $v=2$, variable $k$. Each data point is averaged over 20 runs. See Appendix~\ref{sec:moreruntime} for more detailed experimental results.
    }
    \label{fig:runtimecomparison}
\end{figure}

\subsection{Efficiency of \texorpdfstring{$\circDist_p$}{CWp} Computation}\label{sec:runtime}

To evaluate the runtime of computing $\circDist_p$ against $\gmmDist_p$ and $\wassDist_p$, we randomly generate pairs of circuits with Bernoulli input distributions with a given variable scope size $v$ and sum node branching factor $k$. To do this, we first construct a hierarchical scope partitioning given the variable scope size, then construct two PCs following this scope partitioning (implying that they are compatible) with sum node branching factor $k$. These circuits mirror the structures learned using a state-of-the-art structure learning algorithm---HCLT~\citep{hclt}---where $k$ corresponds to the block size of the HCLT structure. \ac{Estimating $\wassDist_p$ using the Sinkhorn algorithm is done by sampling $n$ data points from each PC and computing the Sinkhorn distance between the samples.}

We implement our algorithm as detailed in Algorithm \ref{algimp:couple-cw} to compute the optimal transport map and value for $\circDist_p$. For our baselines, (i) we implement a PC-to-GMM unrolling algorithm and employ the algorithm proposed by \citet{delon2020wasserstein} to compute $\gmmDist_p$ between the unrolled circuits, and (ii) we also compute $\wassDist_p$ between PCs representing discrete distributions by enumerating the likelihood of every variable assignment for each circuit and solving a linear program to get the exact Wasserstein distance.

The results are summarized in Figure~\ref{fig:runtimecomparison}. We observe that the time to compute $\wassDist_p$ remains effectively constant in $k$ but grows exponentially in $v$, as the number of joint assignments and thus the size of the linear program grows exponentially in $v$. The time to compute $\gmmDist_p$ grows exponentially in both $v$ and $k$, while the time to compute $\circDist_p$ grows linearly in $v$ and quadratically in $k$. Therefore, $\circDist_p$ is the only metric that can be efficiently computed for any circuit. 

We further note that the exponential increase in problem size for $\gmmDist_p$ and $\wassDist_p$ introduces both numerical stability and out-of-memory issues on a machine with 256Gb of RAM when computing the metrics for larger circuits, as an exponentially larger linear optimization problem must be solved. Conversely, this is never the case for computing $\circDist_p$, as the required linear programs remain small with only $k^2$ variables and do not depend on the circuit size. See Figure~\ref{fig:numerical} for more details.

\begin{figure}
    \centering
    \includegraphics[width=0.5\linewidth]{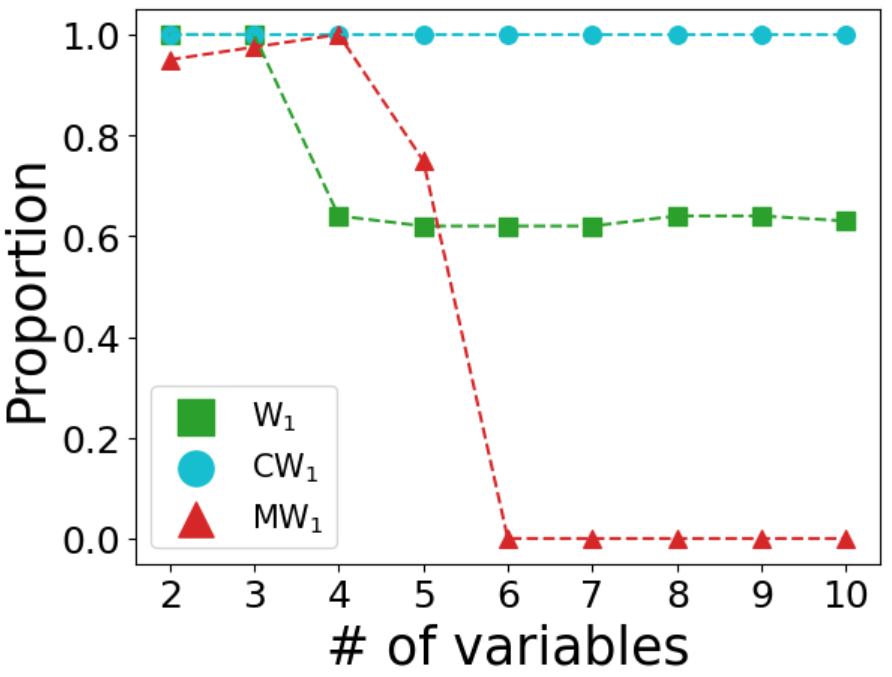} 
    \caption{\small Proportion of instances that $\circDist_1$, $\gmmDist_1$, and $\wassDist_1$ could be solved without numerical stability or OOM issues. Note that only $\circDist_1$ could be computed exactly for every circuit pair.}
    \label{fig:numerical}
\end{figure}

The choice of Bernoulli leaves is intentional to allow $\wassDist_p$ to be used as a baseline; while the choice of leaf distribution is irrelevant for $\gmmDist_p$ and $\circDist_p$ computation time, choosing a leaf distribution with even 3 categories would render most of the $\wassDist_p$ experiments performed above impractical due to the requirement of enumerating every variable assignment for the circuit. Furthermore, it is impossible to compute $\wassDist_p$ in this manner using \emph{any} continuous distribution for the leaves, while $\gmmDist_p$ and $\circDist_p$ can still be computed in the same amount of time. See Appendix~\ref{sec:moreruntime} for runtime experiments with circuits with Gaussian input distributions.

\subsection{\texorpdfstring{$\circDist_p$}{CWp} as a Proxy Metric for \texorpdfstring{$\wassDist_p$}{Wp}}

It is well-established that $\wassDist_p$ is a useful metric for comparing distances between distributions despite it often being impractical to compute; thus, a natural question is whether this utility extends to $\circDist_p$.

\paragraph{$\circDist_p$ Between Randomly-Generated Circuits}
\ac{We first evaluate $\circDist_p$ between pairs of small, randomly-generated circuits for which $\wassDist_p$ can be computed exactly in order to determine how similar our metric $\circDist_p$ ranks the distances between these pairs when compared to $\wassDist_p$. Intuitively, the closer these rankings are, the better $\circDist_p$ is as a proxy for $\wassDist_p$. Thus, for each of the $n$ circuit pairs randomly generated using the same procedure as in Section~\ref{sec:runtime}, we compute two rankings---one ranking by $\circDist_1$ and the other by $\wassDist_1$---and compute the Kendall rank correlation coefficient $\tau$ between the two rankings.}
Across all settings for $v$ and $k$ for which $\wassDist_1$ was computable without running into numerical stability issues (30 out of 45 instances), the Kendall rank correlation coefficient $\tau$ was at least 0.52, with the mean of $0.70$ and variance of $0.007$. This indicates a strong correlation between the distance rankings of $\circDist_1$ and $\wassDist_1$.

\ac{We also investigate the Pearson correlation between $\circDist_1$ and $\wassDist_1$, and compare it to the Pearson correlation between the Sinkhorn distance and $\wassDist_1$. Across all settings for $v$ and $k$ not encountering numerical stability issues, the average correlation between $\circDist_1$ and $\wassDist_1$ was \textbf{0.90}, while the average correlation between the Sinkhorn distance and $\wassDist_1$ was \textbf{0.29} (with the highest correlation being 0.69).}

\paragraph{$\circDist_p$ Between Circuits Learned from Data}
\begin{figure}
    \centering
    \includegraphics[width=0.8\linewidth]{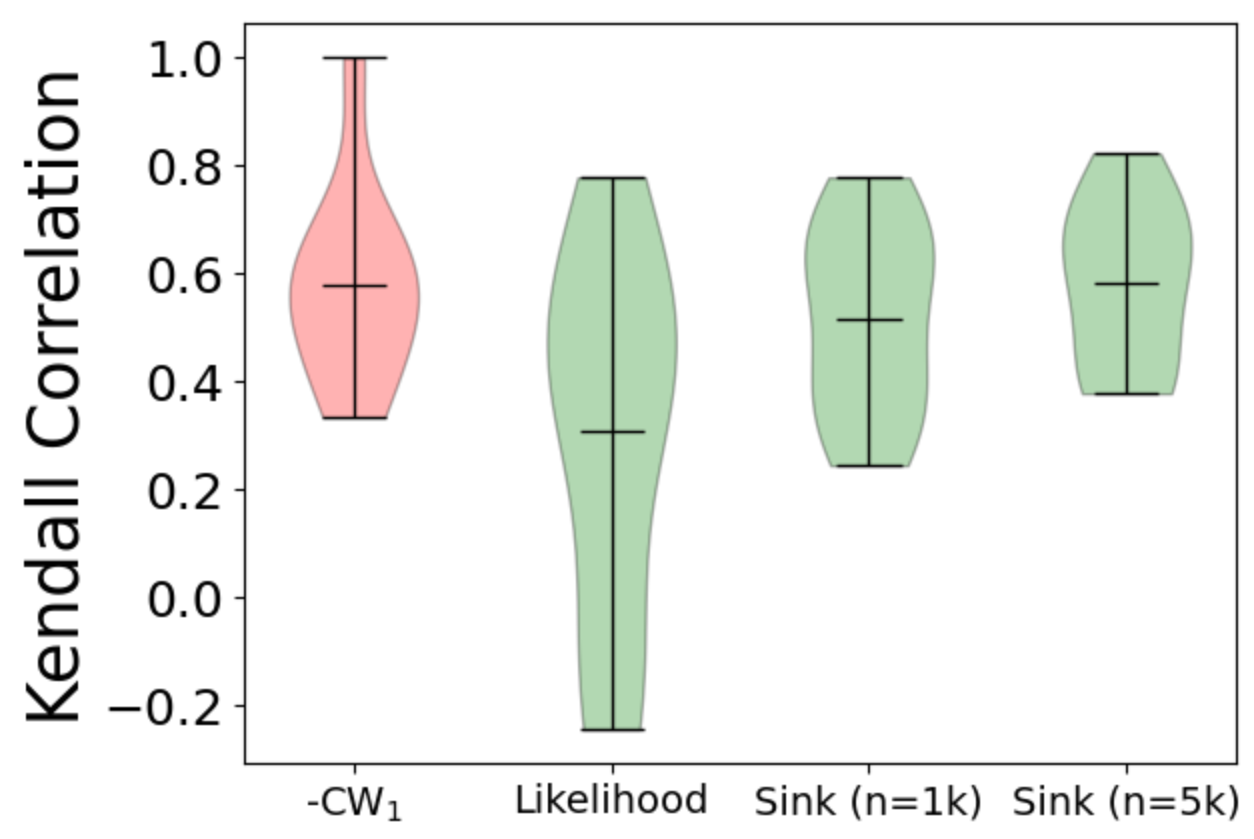} 
    \caption{\small \ac{Distributions of Kendall correlation coefficients between Cosine Similarity and $-\circDist_1$ (far left), likelihood (center left), negative Sinkhorn distance computed with $n=1000$ samples (center right), and negative Sinkhorn distance computed with $n=5000$ samples (far right). Higher is better.}
    }
    \label{fig:mnistexps}
\end{figure}

To support the efficacy of our metric and algorithm on large, high-dimensional PCs beyond randomly-generated circuits, we also evaluate on PCs learned on the MNIST dataset~\citep{mnist}---a 784-dimensional handwritten digits image dataset. Specifically, we first partition the dataset into 10 digit classes, then learn one PC per class using the HCLT structure learning algorithm~\citep{hclt} with a fixed block size of 4 (resulting in each circuit having over 11k edges) and EM parameter learning~\citep{circuitem} implemented in PyJuice~\citep{liu2024scalingtractableprobabilisticcircuits}.
\ac{Between each pair of circuits, we then compute $\circDist_1$ (taking $<5s$ per pair), the Sinkhorn distance using both 1k and 5k samples per pair (taking 1.62 and 8.51 seconds per pair, on average)}, as well as the average likelihood of each class-partitioned dataset for each PC. Lastly, we also compute the average \textit{cosine similarity (CS)} between dataset partitions.
\ac{Figure~\ref{fig:mnistexps} illustrates how closely correlated (according to Kendall rank correlation coefficient) the ranking by $-\circDist_1$, negative Sinkhorn distance (computed with both 1k and 5k samples), and average likelihood are to the ranking based on average CS.} We observe a significantly stronger and consistently positive correlation between $\circDist_1$ and CS rankings, when compared to that between likelihood and CS rankings. \ac{Furthermore, we observed that the variance of Sinkhorn distance is sensitive to the number of samples taken from each PC, with mean distances of 2856.2 and 2841.2 and mean standard deviations of 53.4 and 87.0 using 1k and 5k samples per pair, respectively.}

\subsection{Case Study: Color Transfer using Circuit Transport Maps}

\begin{figure}
    \centering
    \includegraphics[width=0.88\linewidth]{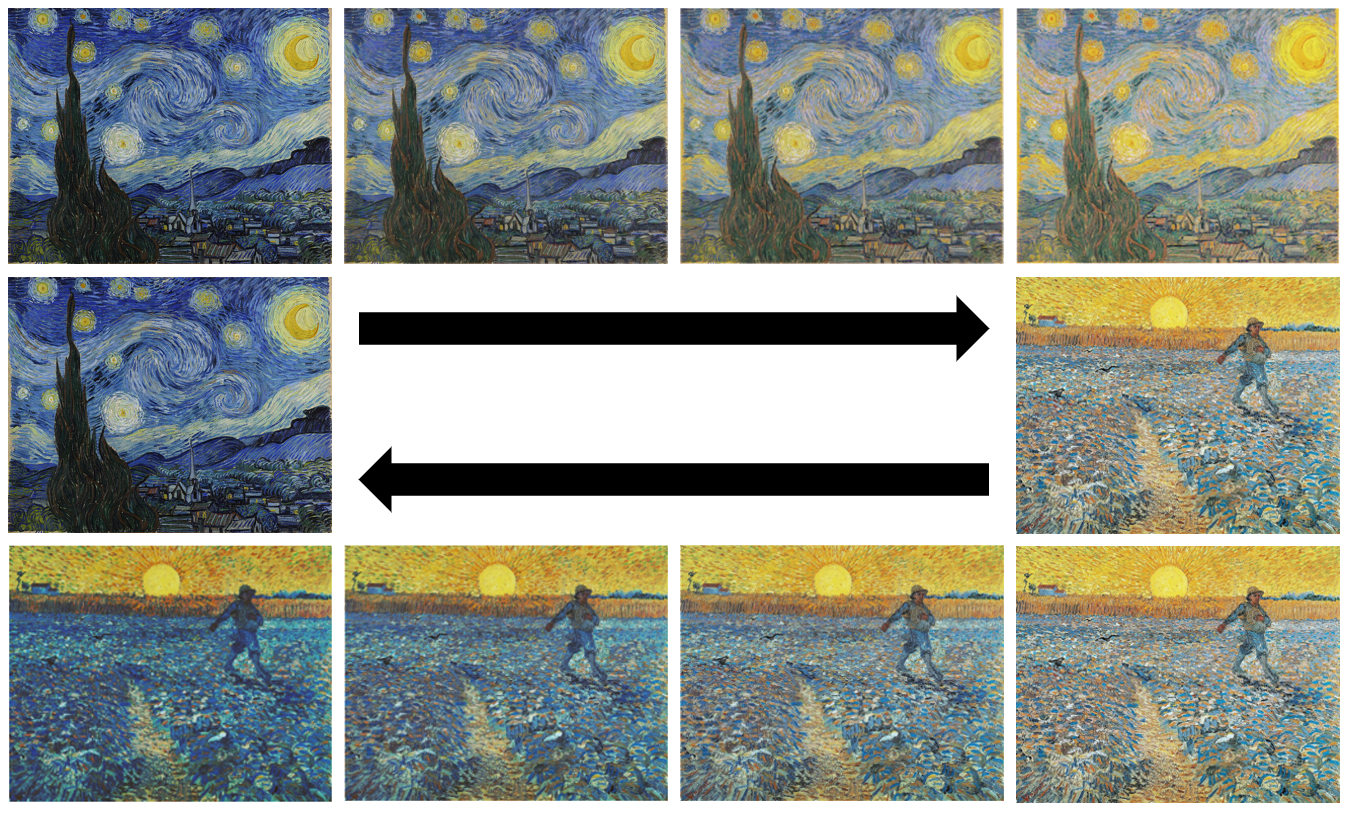}
    \caption{\small Color transfer between images along geodesics using coupling circuits, for $t=0, \frac{1}{3}, \frac{2}{3}$, $1$ in the direction of arrows.}
    \label{fig:colortransfer}
\end{figure}

To evaluate the practicality of transport maps given by our coupling circuits, we adopt an application of optimal transport shown by \citet{delon2020wasserstein}, whereby we transport the \textit{color histogram}---the 3-dimensional probability distribution of pixel color values---of image $a$ to that of another image $b$. To do this, we learn compatible PCs $P(\X)$ and $Q(\Y)$ over the color distributions of images $a$ and $b$, compute the optimal coupling circuit $C(\X,\Y)$, and transport each pixel with color value $\x$ to the corresponding pixel $\y=\Ex_C[\Y|\X=\x]$ (which can be computed tractably). Furthermore, we can transport each pixel value along the geodesic with points in $\x + t(\y-\x)$ for $t \in [0,1]$. By doing this for a fixed $t$ for every pixel in the image, we can interpolate between two images according to our transport map in the color domain. See Figure~\ref{fig:colortransfer} for two examples, and Appendix~\ref{sec:morecolortransfer} for additional examples.

\subsection{Wasserstein-Minimization for Parameter Learning}

\begin{table}
    \centering
    \scalebox{0.87}{
\begin{tabular}{ |c||c|c|c|c|c|c| } 
\hline
& \multicolumn{2}{|c|}{EM Circuit} & \multicolumn{2}{|c|}{Deterministic WM} & \multicolumn{2}{|c|}{Stochastic WM} \\
\hline
$k$ & $\wassDist_2$ & BPD & $\wassDist_2$ & BPD & $\wassDist_2$ & BPD \\
\hline
\hline
4 & 32631 & \textbf{1.414} & 32766 & 1.495 & \textbf{29963} & 1.532 \\ 
16 & 32873 & \textbf{1.242} & 32751 & 1.465 & \textbf{29984} & 1.509 \\ 
64 & 33264 & \textbf{1.192} & 32749 & 1.458 & \textbf{30999} & 1.485 \\ 
128 & 33737 & \textbf{1.175} & 32749 & 1.455 & \textbf{31483} & 1.474 \\ 
256 & 34974 & \textbf{1.172} & 32528 & 1.459 & \textbf{32520} & 1.459\\ 
\hline
\end{tabular}
}
    \caption{\small Comparison of Wasserstein objective value and bits-per-dimension (BPD)---which is proportional to negative log-likelihood---between circuits learned via EM and two variations of WM (our approach), lower is better. The lowest value for each circuit size is bolded.
    }
    \label{table:emvswm}
\end{table}

To determine the performance of our proposed Wasserstein-Minimization algorithm on density estimation tasks, we consider learning the parameters of circuits of various sizes from the MNIST benchmark dataset~\citep{mnist}. We first generate the structure of circuits using the HCLT algorithm~\citep{hclt}, 
varying the ``block size'' for different numbers of parameters.
We then learn three sets of circuit parameters per structure per block size: one using mini-batch EM~\citep{circuitem}, another using our proposed deterministic WM algorithm, and a third using our stochastic WM algorithm with $p=0.1$.
All experiments were ran on a single NVIDIA L40s GPU. We refer to Appendix~\ref{sec:parameterlearningexps} for more details.

Table~\ref{table:emvswm} summarizes the results.
We observe that our algorithms perform nearly as well as EM for learning small circuits (block size 4) benchmarked on likelihoods, and outperform EM at learning circuits with a low Wasserstein objective. However, as the size of the circuits increase, the performance of our algorithms quickly stagnates; empirically, our WM approaches do not seem to take full advantage of the larger parameter space of larger models, with models orders of magnitude larger having better but still comparable performance to their smaller counterparts when computing likelihoods and a higher Wasserstein objective.

\section{Conclusion}

This paper studied the optimal transport problem for probabilistic circuits. We introduced a Wasserstein-type distance $\circDist_p$ between two PCs and proposed an efficient algorithm that computes the distance as well as associated optimal transport plan in quadratic time in the size of the input circuits, provided that they have compatible structures.
We showed that $\circDist_p$ always upper-bounds the true Wasserstein distance, and that---when compared to the naive application of an existing algorithm for computing a Wasserstein-type distance between GMMs to PCs---the former is exponentially faster to compute between circuits. Lastly, we propose an iterative algorithm to minimize the empirical Wasserstein distance between a circuit and data, suggesting an alternative, viable approach to parameter estimation for PCs which is often done using maximum-likelihood estimation. 

We consider this work an initial stepping stone towards a deeper understanding of optimal transport theory for probabilistic circuits. Future directions include exploring more expressive formulations of coupling circuits to obtain a tighter bound on Wasserstein distance---such as relaxing the node-by-node parameter constraints to only require that the whole circuit matches marginal distributions to the original circuits. Our work also leaves open the possibility of extending the marginal-preserving properties of coupling circuits to the multimarginal setting for multimarginal generative modeling with PCs, and computing Wasserstein barycenters for PCs. Moreover, we envision that the tractable computation of a Wasserstein-type distance and transport plan between expressive models such as PCs can lead to further development in various Wasserstein-based machine learning approaches.

\begin{acknowledgements}
    This work was partially supported by a gift from Cisco University Research Program. The authors would like to thank Geunyeong Byeon for helpful discussions on Wasserstein learning.
\end{acknowledgements}

\bibliography{main}

\newpage

\onecolumn

\title{Optimal Transport for Probabilistic Circuits\\(Supplementary Material)}
\maketitle

\appendix\section{Algorithm for Minimum Wasserstein Parameter Estimation}\label{alg:minwass}

Our proposed algorithm is broadly divided into two steps: an inference step and a minimization step. These steps are performed iteratively until model convergence. The inference step populates a cache, which stores the expected distance of each data point at each node in the circuit. This inference step is done in linear time in a bottom-up recursive fashion, making use of the cache for already-computed results. This is provided in algorithm \ref{algimp:inf}.

The minimization step is done top-down recursively, and seeks to route the data at a node to its children in a way that minimizes the total expected distance between the routed data at each child and the sub-circuit. The root node is initialized with all data routed to it. At a sum node, each data point is routed to the child that has the smallest expected distance to it (making use of the cache from the inference step), and the edge weight corresponding to a child is equal to the proportion of data routed to that child; at a product node, the data point is routed to both children. Input node parameters are updated to reflect the empirical distribution of the data routed to that node. The minimization step is thus also done in linear time, and we note that this algorithm guarantees a non-decreasing objective function (see Appendix \ref{pf:monoton} for a proof). The algorithm for this is provided in algorithm \ref{algimp:learn}.

\begin{algorithm}
\caption{\textsc{Inference}$(n, D)$: returns a cache storing the distance between each data point $d_j\in D$ and each sub-circuit rooted at $n$, where $n$ has children $c_i$. For conciseness, we omit checking for cache hits}\label{algimp:inf}

\begin{algorithmic}[1]

\For{$c_i\in n$.children}

\State \textsc{Inference}$(c_i,D)$ \Comment{recursively build cache}

\EndFor

\If{$n$ is a product node}

\For{$d_j\in D$}

\State cache$[n,d_j]\gets \sum_i$cache$[c_i,d_j]$

\EndFor

\EndIf

\If{$n$ is a sum node}

\For{$d_j\in D$}

\State cache$[n,d_j]\gets \sum_i \theta_i$cache$[c_i,d_j]$

\EndFor

\EndIf

\If{$n$ is an input node}

\For{$d_j\in D$}

\State cache$[n,d_j]\gets dist(n,d_j)$ \Comment{here, $dist(n, d_j)$ is the expected distance between $n$ and $d_j$}

\EndFor

\EndIf

\State \textbf{return} cache
\end{algorithmic}
\end{algorithm}

\begin{algorithm}
\caption{\textsc{Learn}$(n, D, \text{cache})$: learns the parameters of circuit rooted at $n$ on data points $d_j\in D$}\label{algimp:learn}
\begin{algorithmic}[1]
    
\If{not all parents of $n$ have been learned}

\State \textbf{return} \Comment{We only call this method on nodes who's parents have all been learned}

\EndIf

\If{$n$ is a product node}

\For{$c_i\in n.$children}

\State \text{routing}$[c_i]\gets \text{routing}[n]$ \Comment{products route their data to their children}

\EndFor

\EndIf

\If{$n$ is a sum node}

\State $\forall \theta_i,$ $\theta_i \gets 0$ \Comment{zero out parameters}

\For{$d_j\in \text{routing}[n]$} 
\State \Comment{route data points at current node to children}
\State $c_i\gets \arg\min_{c_i} \text{ cache[}c_i,d_j]$ \Comment{$c_i$ is the child node of $n$ for which $d_j$ has the lowest distance}
\State routing[$c_i$] $\gets d_j$ \Comment{route $d_j$ to $c_i$}
\State $\theta_i \gets \theta_i + \frac{1}{|\text{routing}[n]|}$ \Comment{update parameter weight}
\EndFor

\EndIf

\If{$n$ is an input node}

\State $n$.parameters $\gets$ parameters matching empirical distribution of routing$[n]$

\EndIf
\end{algorithmic}

\end{algorithm}

\section{Proofs}

\subsection{Hardness Proof of the \texorpdfstring{$\infty$}{∞}-Wasserstein Distance Between Circuits}\label{proof:hardness}

\rethm{thm:w-hardness}{}{
Suppose $P$ and $Q$ are probabilistic circuits over $n$ Boolean variables. Then computing the $\infty$-Wasserstein distance between $P$ and $Q$ is coNP-hard, even when $P$ and $Q$ are \textit{deterministic} and \textit{structured-decomposable}.
}
\begin{proof}
     We will prove hardness by reducing the problem of deciding equivalence of two DNF formulas, which is known to be coNP-hard, to Wasserstein distance computation between two compatible PCs.

     Consider a DNF $\alpha$ containing $m$ terms $\{\alpha_1,\dots,\alpha_m\}$ over Boolean variables $\X$. We will construct a PC $P_\alpha$ associated with this DNF as follows. For each term $\alpha_i$, we construct a product of input nodes---one for each $X\in\X$ whose literal appears in term $\alpha_i$, $\mathds{1}[X=1]$ for a positive literal and $\mathds{1}[X=0]$ for negative. Then we construct a sum unit with uniform parameters over these products as the root of our PC: $P_{\alpha} = \sum_{i=1}^m \frac{1}{m} P_{\alpha_i}$. We can easily smooth this PC by additionally multiplying $P_{\alpha_i}$ with a sum node $\frac{1}{2}\mathds{1}[X=0] + \frac{1}{2}\mathds{1}[X=1]$ for each variable $X$ that does not appear in $\alpha_i$. Furthermore, note that every product node in this circuit fully factorizes the variables $\X$, and thus the PC is trivially compatible with any decomposable circuit over $\X$ and in particular with any other PC for a DNF over $\X$, constructed as above.

     Clearly, the above PC $P_\alpha$ assigns probability mass only to the models of $\alpha$. In other words, for any $\x\in\binary^n$, $P_\alpha(\x) > 0$ iff $\x \models \alpha$ (i.e. there is a term $\alpha_i$ that is satisfied by $\x$).
\end{proof}

\subsection{Proof of the Marginal-Matching Properties of Coupling Circuits}\label{proof:marginalmatching}

\begin{prop}\label{prop:marginalmatching}
    Let $P$ and $Q$ be compatible PCs. Then any feasible coupling circuit $C$ as defined in Def. \ref{def:couplingcircuits} matches marginals to $P$ and $Q$.
\end{prop}
\begin{proof}
We will prove this by induction. Our base case is two corresponding input nodes $n_1,n_2 \in P,Q$. The sub-circuit in $C$ rooted at the coupling of $n_1$ and $n_2$ is an input node representing the optimal transport plan between $n_1$ and $n_2$, which clearly matches marginals to $n_1$ and $n_2$.

Now, let $n_1$ and $n_2$ be arbitrary corresponding nodes in $P$ and $Q$, and assume that the coupling circuits for all children of the two nodes match marginals. We then have two cases:

\paragraph{Case 1: $n_1,n_2$ are product nodes}
Since the circuits are compatible, we know that $n_1$ and $n_2$ have the same number of children---let the number of children be $k$. Thus, let $c_{1,i}$ represent the $i$'th child of $n_1$, and let $c_{2,i}$ represent the $i$'th child of $n_2$. The coupling circuit of $n_1$ and $n_2$ (denoted $n$) is a product node with $k$ children, where the $i$'th child is the coupling circuit of $c_{1,i}$ and $c_{2,i}$ (denoted $c_i$). 

By induction, the distribution $P_{c_i}(\X,\Y)$ at each child coupling sub-circuit matches marginals to the original sub-circuits: $P_{c_i}(\X)=P_{c_{1,i}}(\X)$, and $P_{c_i}(\Y)=P_{c_{2,i}}(\Y)$. $n_1$ and $n_2$ being product nodes means that $P_{n_1}(\X)=\prod_iP_{c_{1,i}}(\X)$ and $P_{n_2}(\Y)=\prod_iP_{c_{2,i}}(\Y)$, so thus $P_n(\X)=\prod_iP_{c_i}(\X)=\prod_iP_{c_{1,i}}(\X)$ and $P_n(\Y)=\prod_iP_{c_i}(\Y)=\prod_iP_{c_{2,i}}(\Y)$. Therefore, $n$ matches marginals to $n_1$ and $n_2$.

\paragraph{Case 2: $n_1,n_2$ are sum nodes}
Let the number of children of $n_1$ be $k_1$ and the number of children of $n_2$ be $k_2$. Let $c_{1,i}$ represent the $i$'th child of $n_1$, and let $c_{2,j}$ represent the $j$'th child of $n_2$. The coupling circuit of $n_1$ and $n_2$ (denoted $n$) is a sum node with $k_1*k_2$ children, where the $(i,j)$'th child is the coupling circuit of $c_{1,i}$ and $c_{2,j}$ (denoted $c_{i,j}$). 

By induction, the distribution $P_{c_{i,j}}(\X,\Y)$ at each child coupling sub-circuit matches marginals to the original sub-circuits: $P_{c_{i,j}}(\X)=P_{c_{1,i}}(\X)$, and $P_{c_{i,j}}(\Y)=P_{c_{2,j}}(\Y)$. $n_1$ and $n_2$ being sum nodes means that $P_{n_1}(\X)=\sum_i\theta_iP_{c_{1,i}}(\X)$ and $P_{n_2}(\Y)=\sum_j\theta_jP_{c_{2,j}}(\Y)$, so thus 
\begin{multline}
    P_n(\X)=\sum_i\sum_j\theta_{i,j}P_{c_{i,j}}(\X)=\sum_i\sum_j\theta_{i,j}P_{c_{1,i}}(\X)=\sum_i\theta_iP_{c_{1,i}}(\X)=P_{n_1}(\X)\\   
    P_n(\Y)=\sum_i\sum_j\theta_{i,j}P_{c_{i,j}}(\Y)=\sum_i\sum_j\theta_{i,j}P_{c_{2,j}}(\Y)=\sum_j\theta_jP_{c_{2,j}}(\Y)=P_{n_2}(\Y)
\end{multline} 

Note that we rewrite $\sum_i\theta_{i,j}=\theta_j$ and $\sum_j\theta_{i,j}=\theta_i$ by the constraints on coupling circuits. Therefore, $n$ satisfies marginal constraints.
\end{proof}

\subsection{Proof of the Metric Properties of \texorpdfstring{$\circDist_p$}{CWp}}\label{proof:metric}

\reprop{prop:metric}{Metric Properties of $\circDist_p$}{For any set $\mathcal{C}$ of compatible circuits, $\circDist_p$ defines a metric on $\mathcal{C}$.
}

\begin{proof}
It is clear that $\circDist_p$ is symmetric since construction of the coupling circuit is symmetric. Furthermore, since $\circDist_p$ upper-bounds $\wassDist_p$, it must also be non-negative. 

If $\circDist_p(P,Q)=0$, then $\wassDist_p(P,Q)=0$ so $P=Q$. Any constraint-satisfying assignment of the parameters of a coupling circuit between $P$ and $P$ would also result in the Wasserstein objective at the root node being $0$, since the base-case computation of $\wassDist_p$ at the leaf nodes would always be zero.

Now, we show that $\circDist_p$ satisfies the triangle inequality. Let $P,Q,R\in \mathcal{C}$ be compatible PCs over random variables $\X, \Y,$ and $\rvars{Z},$ and let $d_1=\circDist_p(P,Q),\text{ } d_2=\circDist_p(P,R), $ and $d_3=\circDist_p(R,Q)$ with optimal coupling circuits $C_1, C_2, $ and $C_3$. We can construct circuits $C_2(\rvars{x}|\rvars{z})$ and $C_3(\rvars{y}|\rvars{z})$ that are still compatible with $C_2$ and $C_3$, since conditioning a circuit preserves the structure. Because all of these are compatible, we can then construct circuit $C(\X,\Y,\rvars{Z})=C_2(\rvars{X}|\rvars{Z})C_3(\rvars{Y}|\rvars{Z})R(\rvars{Z})$. Thus, $C$ is a coupling circuit of $P, Q$, and $R$ such that $C_2(\rvars{x},\rvars{y})=\int C(\rvars{x},\rvars{y}, \rvars{z})d\rvars{z}$ and $C_3(\rvars{y},\rvars{z})=\int C(\rvars{x},\rvars{y}, \rvars{z})d\rvars{x}$. Then we have:
\begin{align*}
\circDist_p(P,Q) &=\int \norm{\x-\y}_p^p C_1(\x, \y)d\x d\y
=\int \norm{(\x - \rvars{z}) - (\y - \rvars{z})}_p^p C(\x, \y, \rvars{z})d\x d\y d\rvars{z}\\
&\leq \int \norm{\x - \rvars{z}}_p^p C_2(\x, \rvars{z}) d\x d\rvars{z} + \int \norm{\rvars{z} - \y}_p^p C_3(\y, \rvars{z}) d\y d\rvars{z}\\
&= \circDist_p(P,R) + \circDist_p(R,Q)
\end{align*}
Thus, $\circDist_p$ satisfies the triangle inequality, which concludes the proof.
\end{proof}

\subsection{Recursive Computation of the Wasserstein Objective}\label{proof:wcomp}

Referring to Definition~\ref{def:cw}, the Wasserstein objective for a given coupling circuit $C(\x, \y)$ is the expected distance between $\x$ and $\y$. Below, we demonstrate that the Wasserstein objective at a sum node that decomposes into $C(\x, \y) = \sum_i\theta_iC_i(\x, \y)$ is simply the weighted sum of the Wasserstein objectives at its children:
\begin{align}
    & \Ex_{C(\x,\y)}[\norm{\x-\y}_p^p]
    = \int \norm{\x-\y}_p^p C(\x,\y) d\x d\y 
    = \int \norm{\x-\y}_p^p \sum_i \theta_i C_i(\x,\y) d\x d\y \nonumber\\
    &= \sum_i \theta_i \int \norm{\x-\y}_p^p C_i(\x,\y) d\x d\y 
    = \sum_i \theta_i \Ex_{C_i(\x,\y)}[\norm{\x-\y}_p^p] \label{eq:sumobjective}
\end{align}

Now, consider a decomposable product node, where $C(\x, \y)=C_1(\x_1, \y_1)C_2(\x_2, \y_2)$ \footnote{We assume for notational simplicity that product nodes have two children, but it is straightforward to rewrite a product node with more than two children as a chain of product nodes with two children each and see that our result still holds.}. Below, we see that the Wasserstein objective at the parent is simply the \emph{sum} of the Wasserstein objectives at its children:
\begin{align}
    & \Ex_{C(\x,\y)}[\norm{\x-\y}_p^p]
    = \int \norm{\x-\y}_p^p C(\x,\y) d\x d\y 
    = \int \norm{\x-\y}_p^p C_1(\x_1,\y_1) C_2(\x_2,\y_2) d\x d\y \nonumber\\
    &= \int (\norm{\x_1-\y_1}_p^p + \norm{\x_2-\y_2}_p^p) \times C_1(\x_1,\y_1) C_2(\x_2,\y_2) d\x_1 d\x_2 d\y_1 d\y_2 \nonumber\\
    &= \left( \int \norm{\x_1-\y_1}_p^p C_1(\x_1,\y_1)  d\x_1 d\y_1 \right) +  \left(\int \norm{\x_2-\y_2}_p^p) C_2(\x_2,\y_2) d\x_2 d\y_2 \right) \nonumber\\
    &= \Ex_{C_1(\x_1,\y_1)}[\norm{\x_1-\y_1}_p^p] + \Ex_{C_2(\x_2,\y_2)}[\norm{\x_2-\y_2}_p^p] \label{eq:prodobjective}
\end{align}
Thus, we can push computation of Wasserstein objective down to the leaf nodes of a coupling~circuit.

\subsection{Hardness Proof of Computing Minimum-Wasserstein Parameters}\label{proof:mwhard}

\rethm{thm:mwhard}{}{    
Suppose $P_{\theta}$ is a smooth and decomposable probabilistic circuit, and $\hat{Q}$ is an empirical distribution induced by a dataset $\data=\{\y^{(k)}\}_{k=1}^n$. Then computing the parameters $\theta$ that minimizes the empirical Wasserstein distance $\wassDist_p^p(P_\theta, \hat{Q})$ (i.e., solving Equation~\ref{eq:emp-wp-dist}) is NP-hard.
}
\begin{proof}
    We will prove this hardness result by reducing $k$-means clustering---which is known to be NP-hard \citep{dasgupta2008hard}---to learning the minimum Wasserstein parameters of a circuit. Consider a set of points $x_1...x_n \in \mathbb{R}^d$ and a number of clusters $k$. We will construct a Gaussian PC $C$ associated with this problem as follows: the root of $C$ is a sum node with $k$ children; each child is a product node with $d$ univariate Gaussian input node children (so each product node is a multivariate Gaussian comprised of independent univariate Gaussians). Minimizing the parameters of $C$ over $x_i$ corresponds to finding a routing of data points $x_i$ that minimizes the total distance between all $x_i$'s and the mean of the multivariate Gaussian child each $x_i$ is routed to. A solution to $k$-means can be retrieved by taking the mean of each child of the root sum node to be the center of each of $k$ clusters.
\end{proof}

\subsection{Proof of the Optimality of Coupling Circuit Parameter Learning in Algorithm \ref{algimp:couple-cw}}\label{proof:exactcw}

\begin{prop}\label{prop:exactcw}
    Suppose $P$ and $Q$ are compatible probabilistic circuits with coupling circuit $C$. Then the parameters of $C$---and thus $\circDist_p$---can be computed exactly in a bottom-up recursive fashion.
\end{prop}
\begin{proof}
    We will construct a recursive argument showing that the optimal parameters of $C$ can be computed exactly. Let $n \in C$ be some node in the coupling circuit $C$ that is the coupling of nodes $n_1$ and $n_2$ in $P$ and $Q$ respectively. Then we have three cases:

    \paragraph{Case 1:  $n$ is an input node}
    By our assumption that the optimal transport plan between $n_1$ and $n_2$ can be computed in constant-time, we let $n$ be an input node with scope $\{\X_k,\Y_k\}$ representing the optimal transport plan between $n_1$ and $n_2$, and $\circDist_p^p(n)=\wassDist_p^p(n_1,n_2)$.

    \paragraph{Case 2:  $n$ is a product node}
    By recursion, $\circDist_p^p(n)=\sum_i\circDist_p^p(c_i)$ for each child $c_i$ of $n$ (see \ref{eq:prodobjective}).

    \paragraph{Case 3:  $n$ is a sum node}
    Let $\theta_{i,j}$ be the parameter corresponding to the product of the $i$-th child of $n_1$ and $j$-th child of $n_2$. We want to solve the following optimization problem $\inf \Ex_{P_n(\X,\Y)}[\norm{\X-\Y}_p^p]$, which can be rewritten as follows:
    \begin{equation}
        \inf \Ex_{P_n(\X,\Y)}[\norm{\X-\Y}_p^p]=\inf\int_{\mathbb{R}^d\times \mathbb{R}^d}\norm{\X-\Y}_p^pP_n(\X,\Y)d\X d\Y
    \end{equation}

    Rewriting the distribution of $n$ as a mixture of its child distributions $c_{i,j}$, we get:
    \begin{equation}
        =\inf_{\theta,P_{i,j}}\int_{\mathbb{R}^d\times \mathbb{R}^d}\norm{\X-\Y}_p^p\sum_{i,j}\theta_{i,j}P_{c_{i,j}}(\X,\Y)d\X d\Y
    \end{equation}

    Due to linearity of integrals, we can bring out the sum:
    \begin{equation}
        =\inf_{\theta,P_{i,j}}\sum_{i,j}\theta_{i,j}\int_{\mathbb{R}^d\times \mathbb{R}^d}\norm{\X-\Y}_p^pP_{c_{i,j}}(\X,\Y)d\X d\Y
    \end{equation}

    Lastly, due to the acyclicity of PCs, we can separate out $\inf_{\theta_i,P_{i,j}}$ into $\inf_{\theta_i}\inf_{P_{i,j}}$ and push the latter infimum inside the sum.
    \begin{equation}
        =\inf_{\theta}\sum_{i,j}\theta_{i,j}(\inf_{P_{i,j}}\int_{\mathbb{R}^d\times \mathbb{R}^d}\norm{\X-\Y}_p^pP_{c_{i,j}}(\X,\Y)d\X d\Y)
    \end{equation}

    Thus, we can solve the inner optimization problem first (corresponding to the optimization problems at the children), and then the outer problem (the optimization problem at the current node). Therefore, a bottom-up recursive algorithm is exact.
\end{proof}

\subsection{Deriving a Closed-Form Solution to the Linear Programs for Parameter Updates}\label{pf:closedform}

For a sum node $s$ with $m$ children $s_1...s_m$ and a dataset with $n$ data points $d_1...d_n$ each with weight $w_j$, we construct a linear program with $m*n$ variables $x_{i,j}$ as follows:
\begin{align*}
&\text{min} \quad
 \sum_{i=1}^m\sum_{j=1}^n\Ex_{s_i}[\norm{\X-d_j}_2^2]x_{i,j} \qquad
\text{s.t.} \quad
\sum_{i=1}^m x_{i,j} = w_j \text{  } \forall j
\end{align*}
Note that the constraints do not overlap for differing values of $j$. Thus, we can break this problem up into $n$ smaller linear programs, each with the following form:
\begin{align*}
&\text{min} \quad
 \sum_{i=1}^m\Ex_{s_i}[\norm{\X-d_j}_2^2]x_{i,j} \qquad
\text{s.t.} \quad
\sum_{i=1}^m x_{i,j} = w_j
\end{align*}
The only constraint here requires that the sum of objective variables is equal to $w_j$. Thus, the objective is minimized when $x_{i,j}$ corresponding to the smallest coefficient takes value $w_j$ and all other variables take value 0. Thus, the solution to the original linear program can be thought of as assigning each data point to the sub-circuit that has the smallest expected distance to it.

\subsection{Proof that the Wasserstein Minimization Algorithm has a Monotonically Decreasing Objective} \label{pf:monoton}

\begin{prop}
    For a circuit rooted at $n$ and dataset $D$ routed to it, the Wasserstein distance between the empirical distribution of $D$ and sub-circuit rooted at $n$ will not increase after an iteration of algorithm \ref{alg:minwass}
\end{prop}

\begin{proof}
    Let $\Ex_n[D]$ denote the Wasserstein distance between the empirical distribution of $D$ and sub-circuit rooted at $n$ before an iteration of algorithm \ref{alg:minwass}, and let $\Ex_{n'}[D]$ denote the distance after an iteration. We will show by induction that $\Ex_{n'}[D] \leq \Ex_{n}[D]$. Our base case is when $n$ is an input node. By setting the parameters of $n$ to as closely match the empirical distribution of $D$ as possible, there is no  parameter assignment with a lower Wasserstein distance to $D$ so thus one iteration of algorithm \ref{alg:minwass} does not increase the objective value.

    Recursively, we have two cases:
    \paragraph{Case 1: $n$ is a product node}
    By the decomposition of the Wasserstein objective, we have that $\Ex_{n}[D]=\sum_i\Ex_{c_i}[D]$, which is $\geq \sum_i\Ex_{c_i'}[D]=\Ex_{n'}[D]$ by induction.

    \paragraph{Case 2: $n$ is a sum node}
    By the decomposition of the Wasserstein objective, we have that $\Ex_{n}[D]=\sum_i\theta_i\Ex_{c_i}[D_i]$ (where $D_i \subseteq D$ is the data routed to $n_i$), which is $\geq \sum_i\theta_i\Ex_{c_i'}[D_i]=\Ex_{n'}[D]$ by induction. Our parameter updates also update each $D_i \rightarrow D_i'$, but that also guarantees that $\Ex_{c_i'}[D_i] \geq \Ex_{c_i'}[D_i']$ since $D_i=D_i'$ is within the feasible set of updates for $D_i$. Thus, $\Ex_{n}[D] \geq \Ex_{n'}[D]$, so therefore the Wasserstein objective is monotonically decreasing.
\end{proof}

\section{Additional Experimental Results}

\subsection{Runtime Experiments for Computing \texorpdfstring{$\circDist_p$}{CWp}} \label{sec:moreruntime}

\begin{figure}
    \centering
    \includegraphics[width=0.35\linewidth]{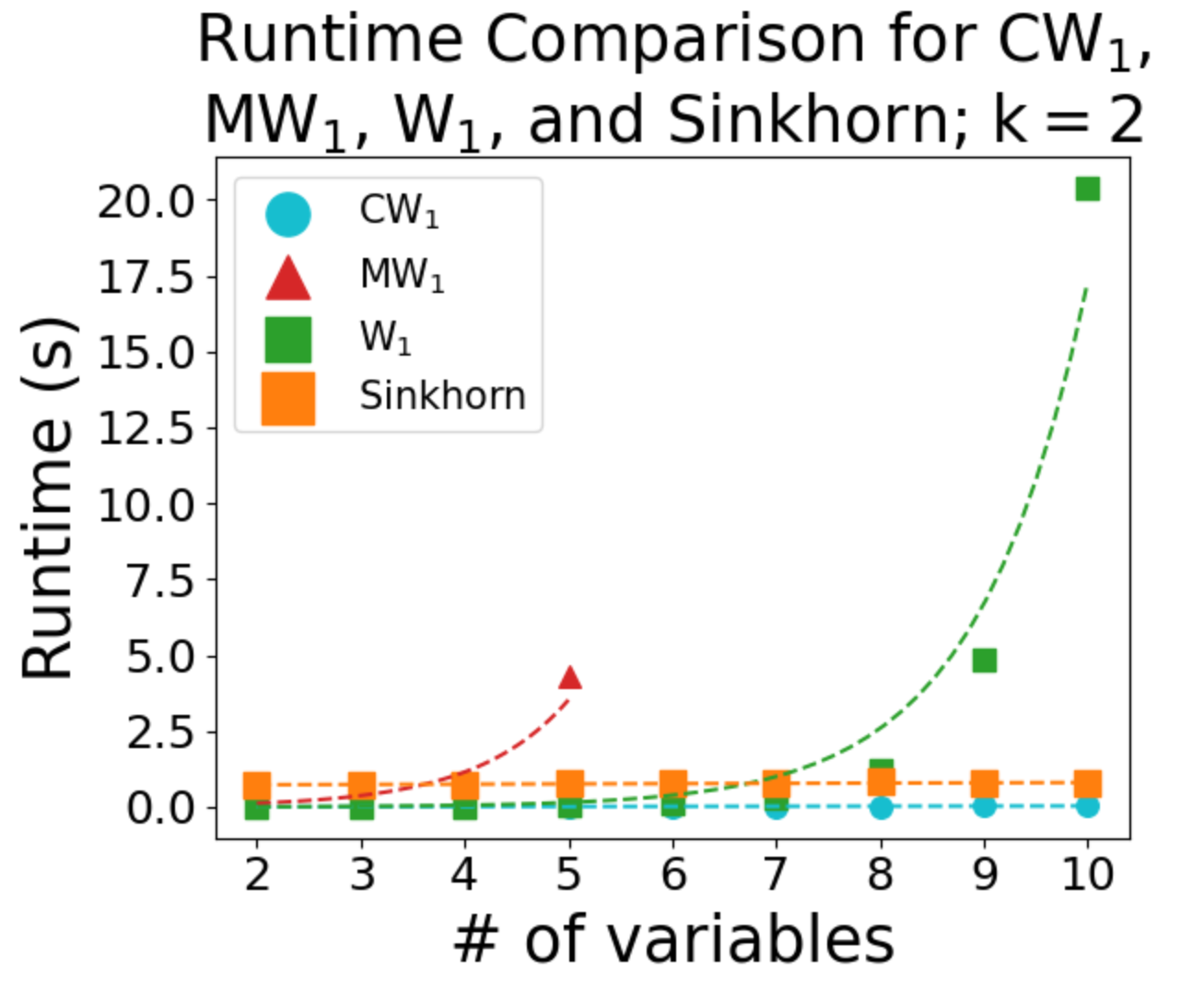} 
    \includegraphics[width=0.35\linewidth]{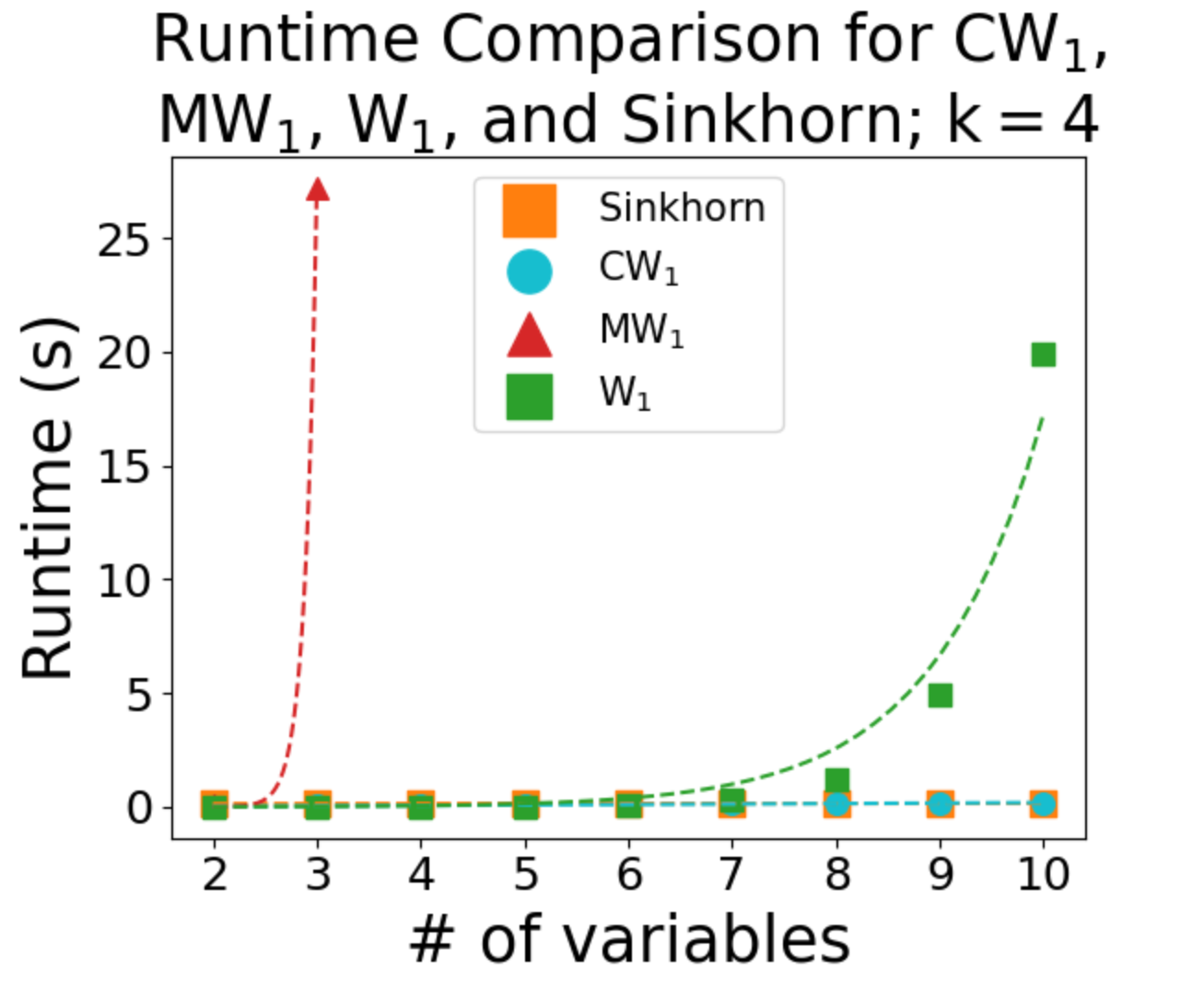} 
    \includegraphics[width=0.35\linewidth]{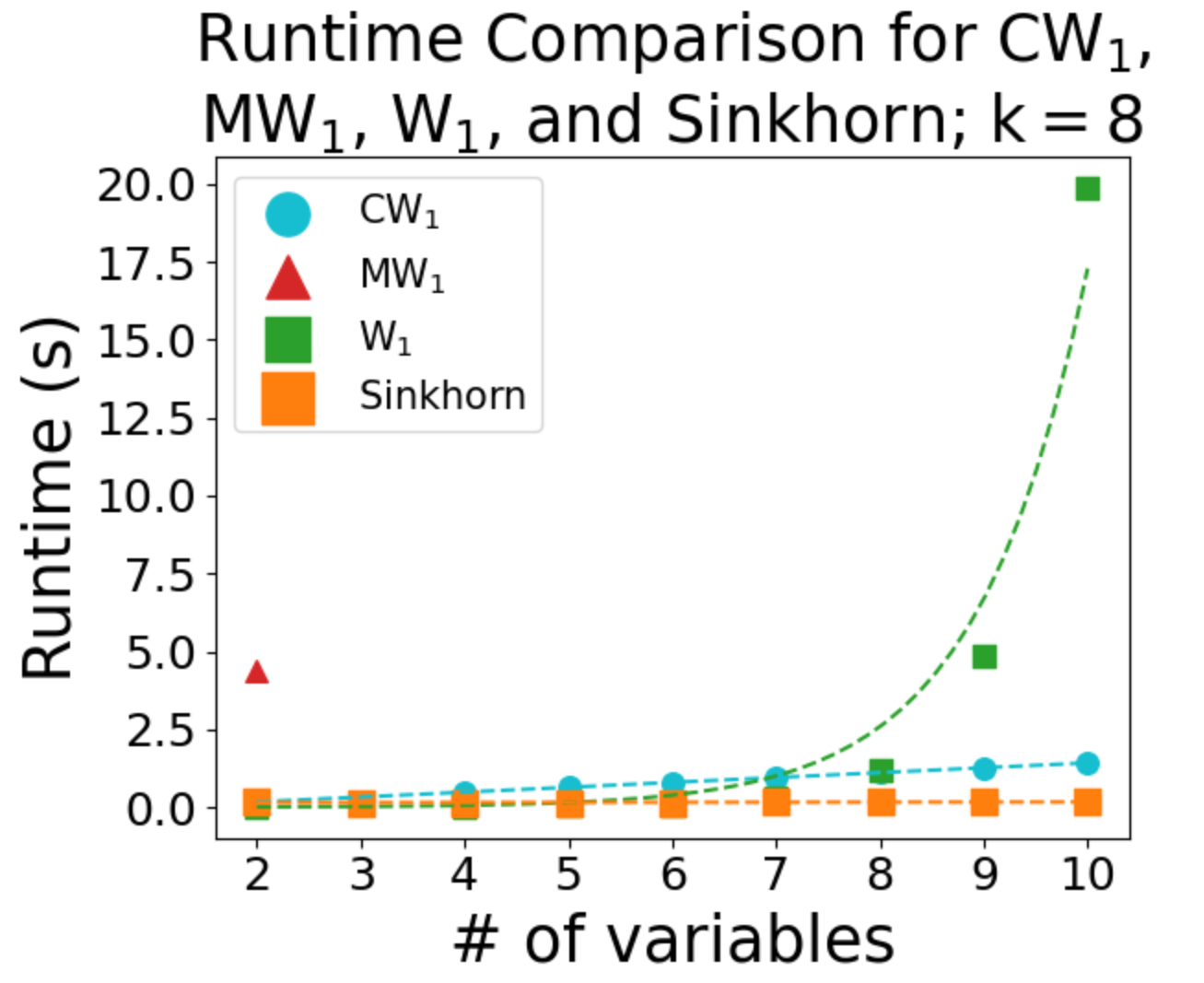} 
    \includegraphics[width=0.35\linewidth]{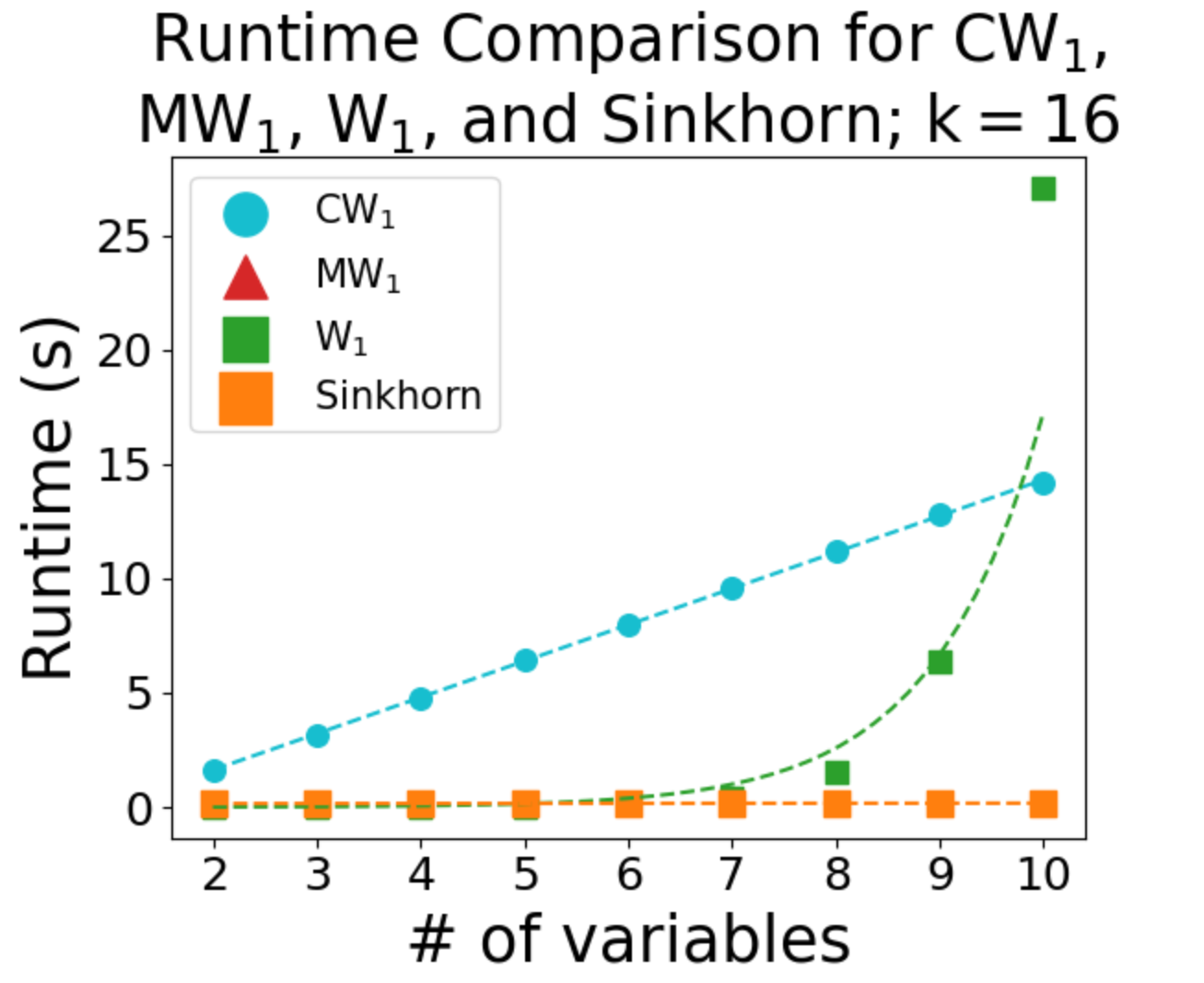} 
    \includegraphics[width=0.35\linewidth]{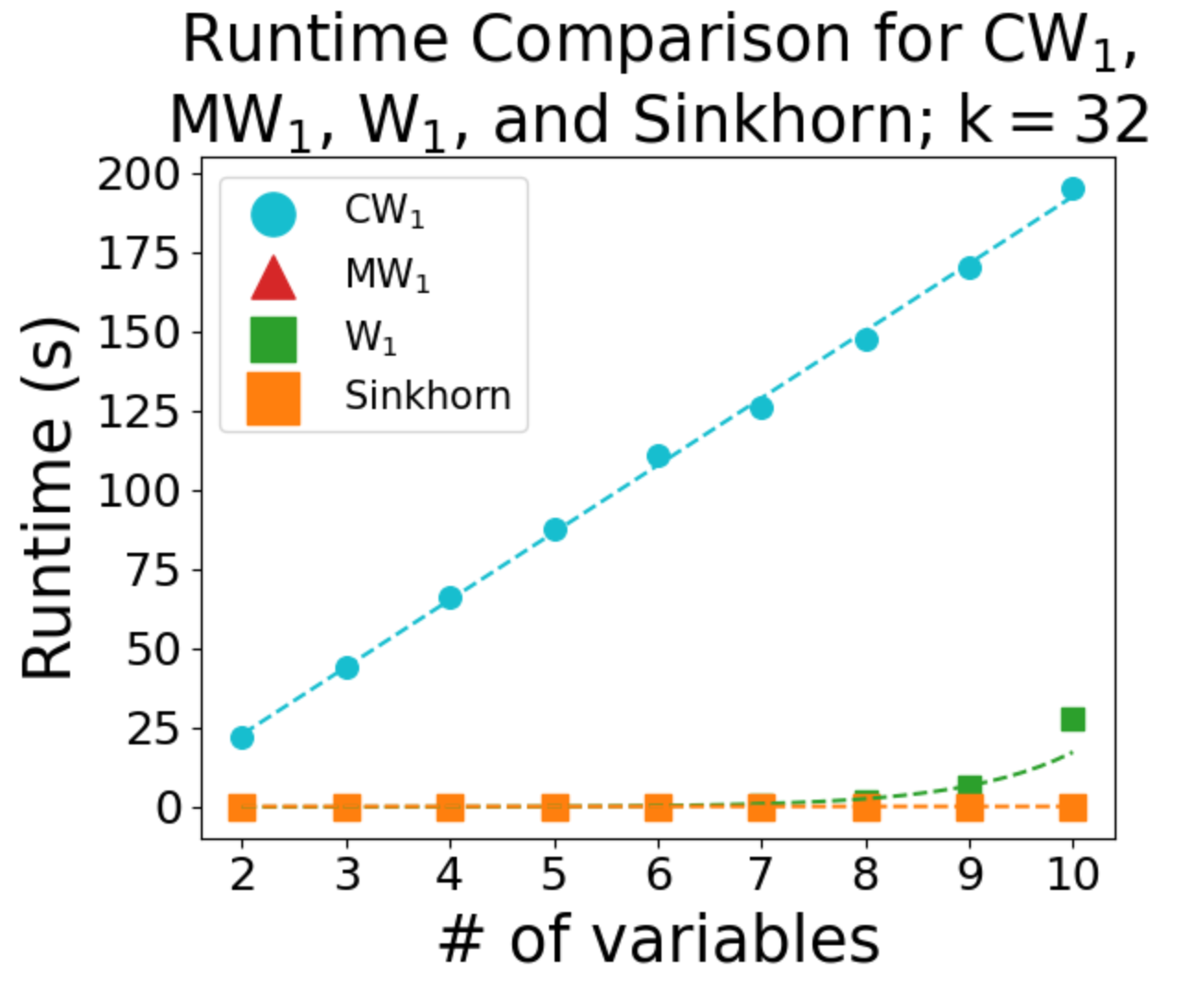} 
    \caption{Runtime for algorithms computing $\circDist_1$, $\gmmDist_1$, and $\wassDist_1$. For each plot, the block size $k$ is kept constant and the variable scope size $v$ is varied.}
    \label{fig:variablescopesize}
\end{figure}

\begin{figure}
    \centering
    \includegraphics[width=0.35\linewidth]{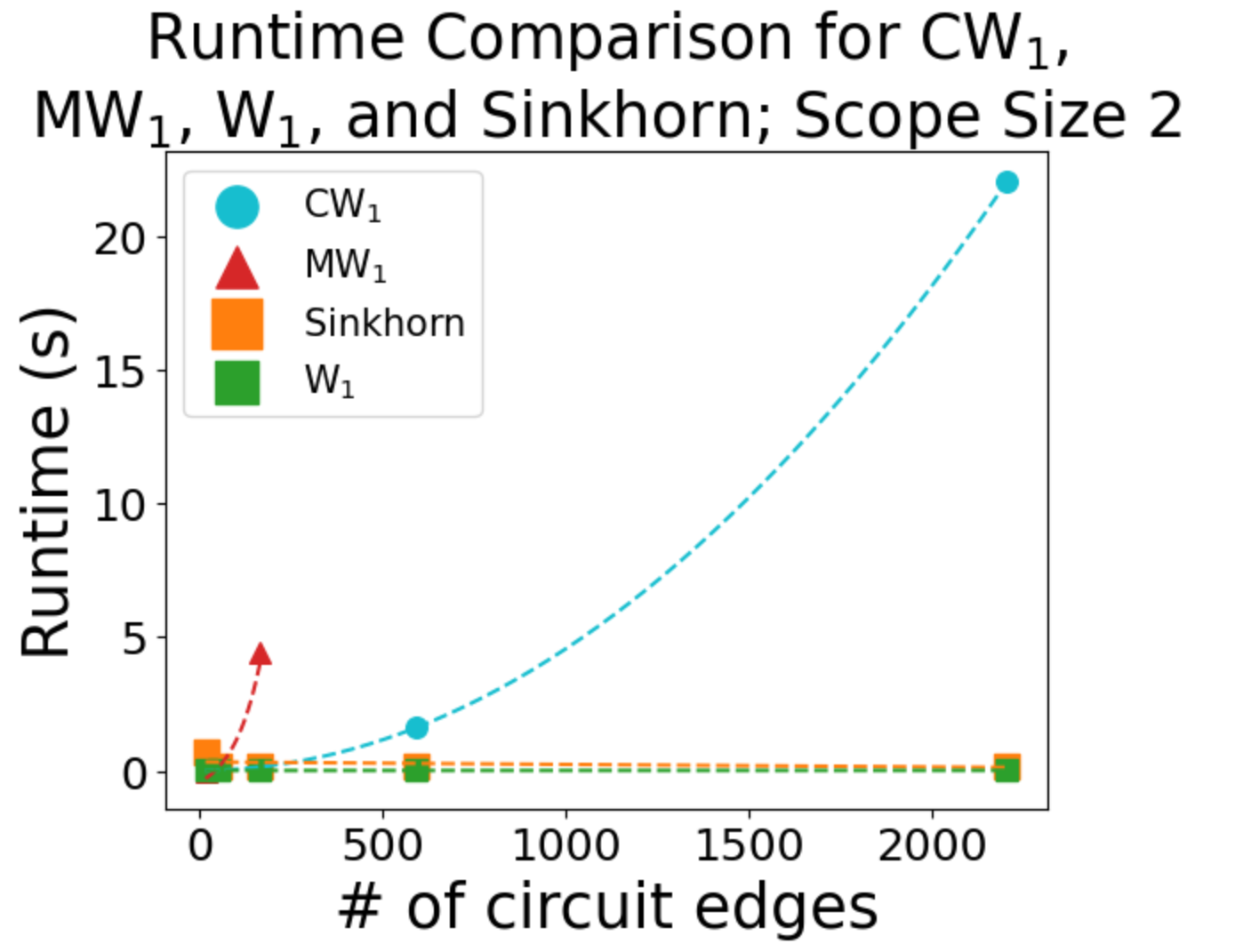} 
    \includegraphics[width=0.35\linewidth]{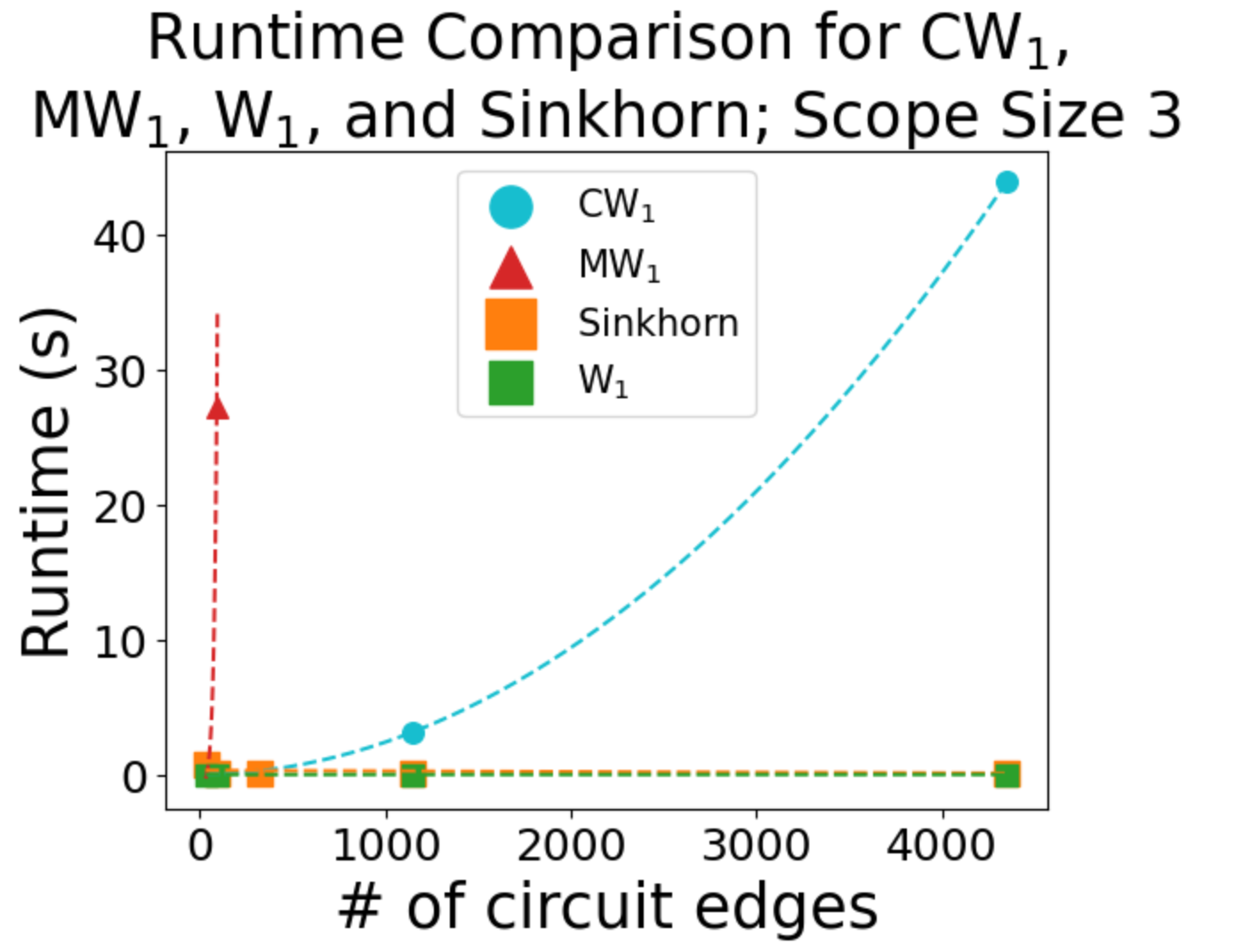} 
    \includegraphics[width=0.35\linewidth]{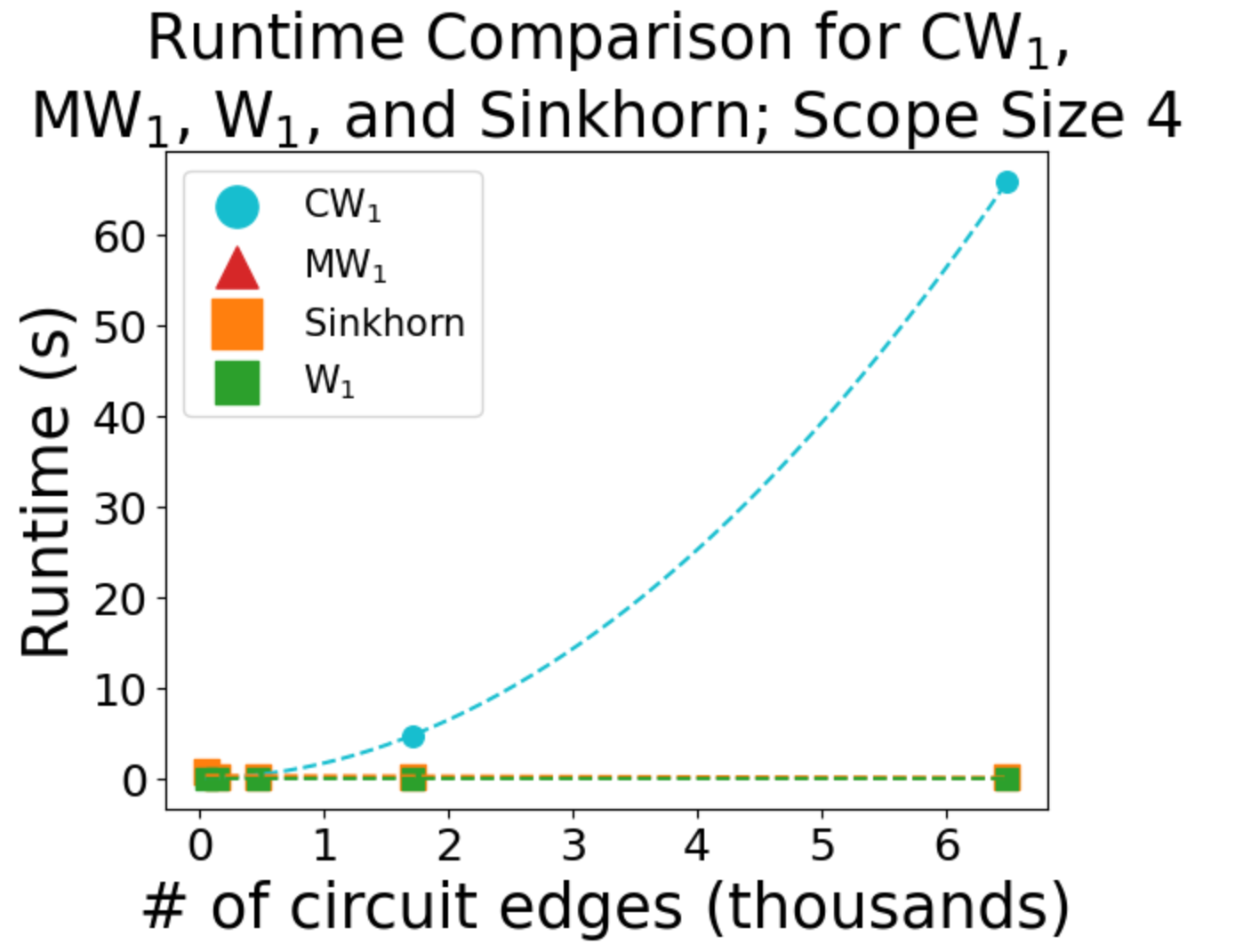} 
    \includegraphics[width=0.35\linewidth]{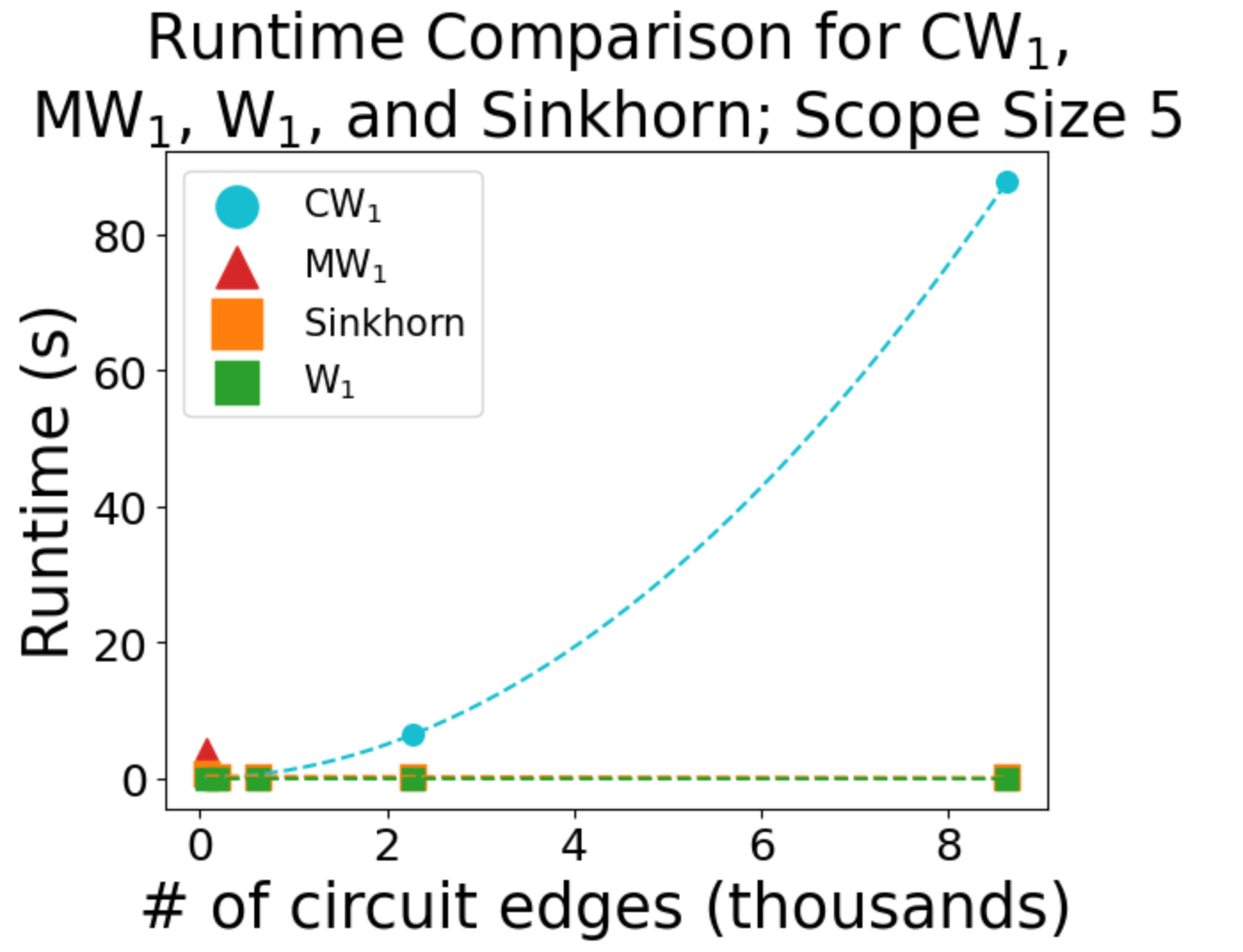}
    \includegraphics[width=0.35\linewidth]{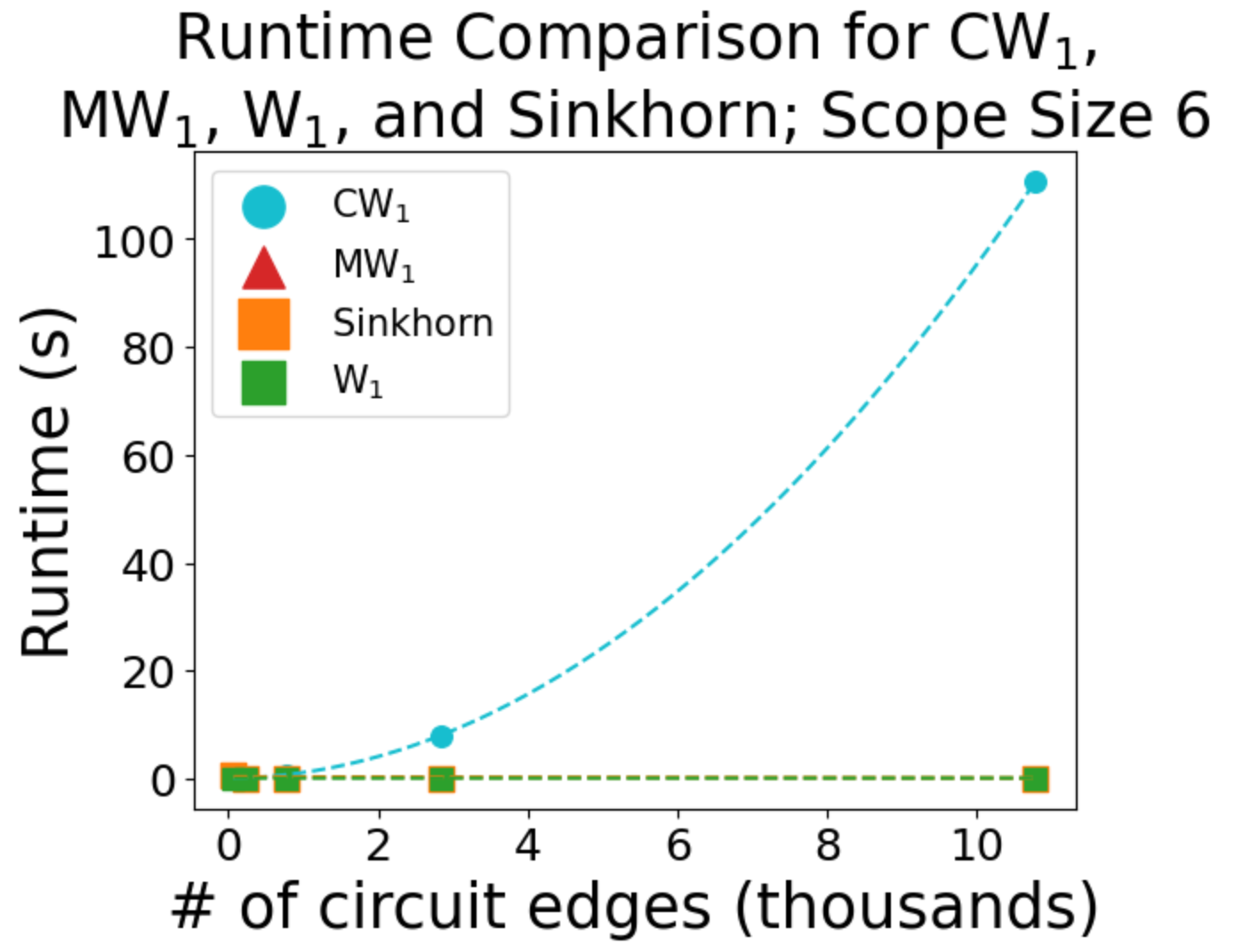} 
    \includegraphics[width=0.35\linewidth]{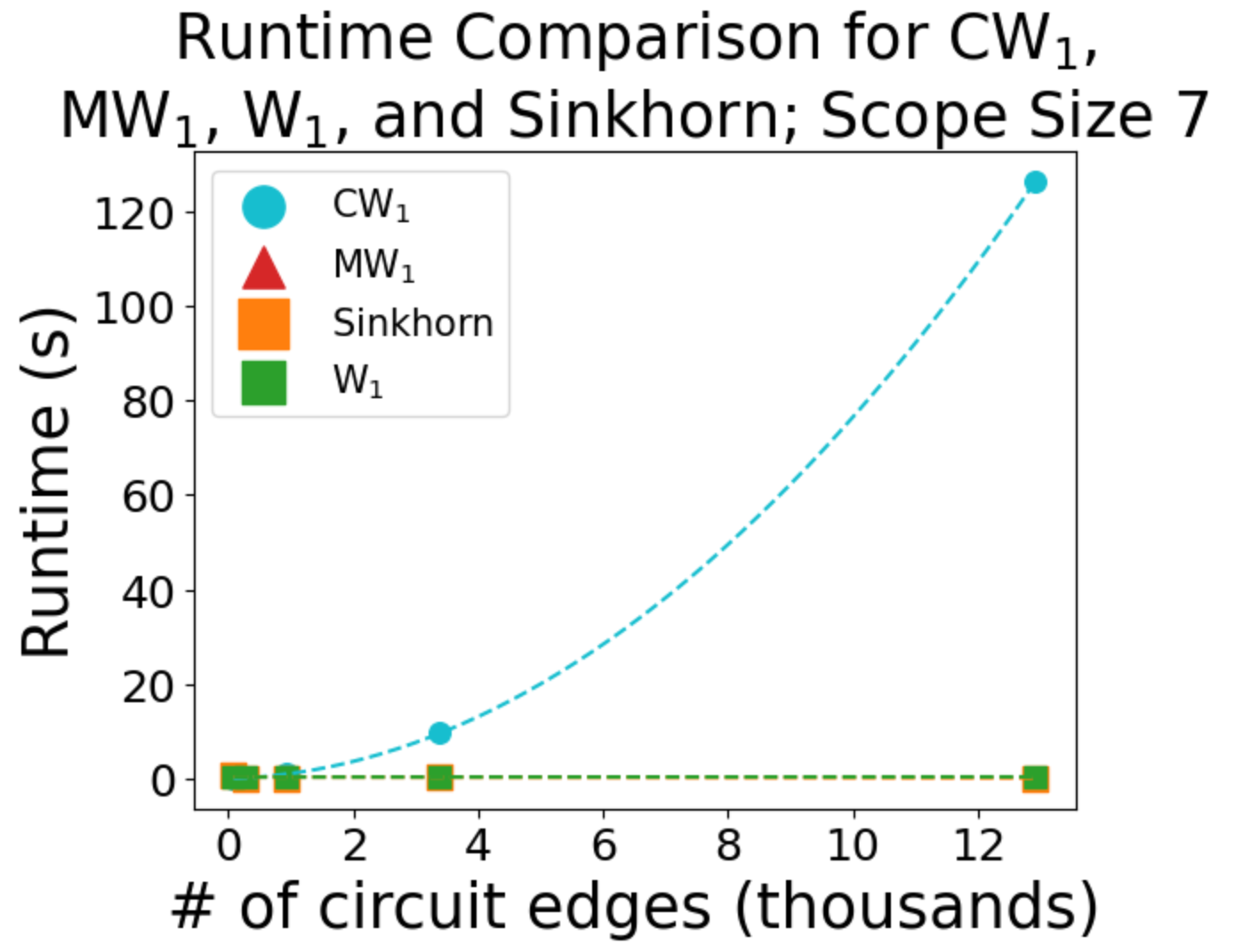} 
    \includegraphics[width=0.35\linewidth]{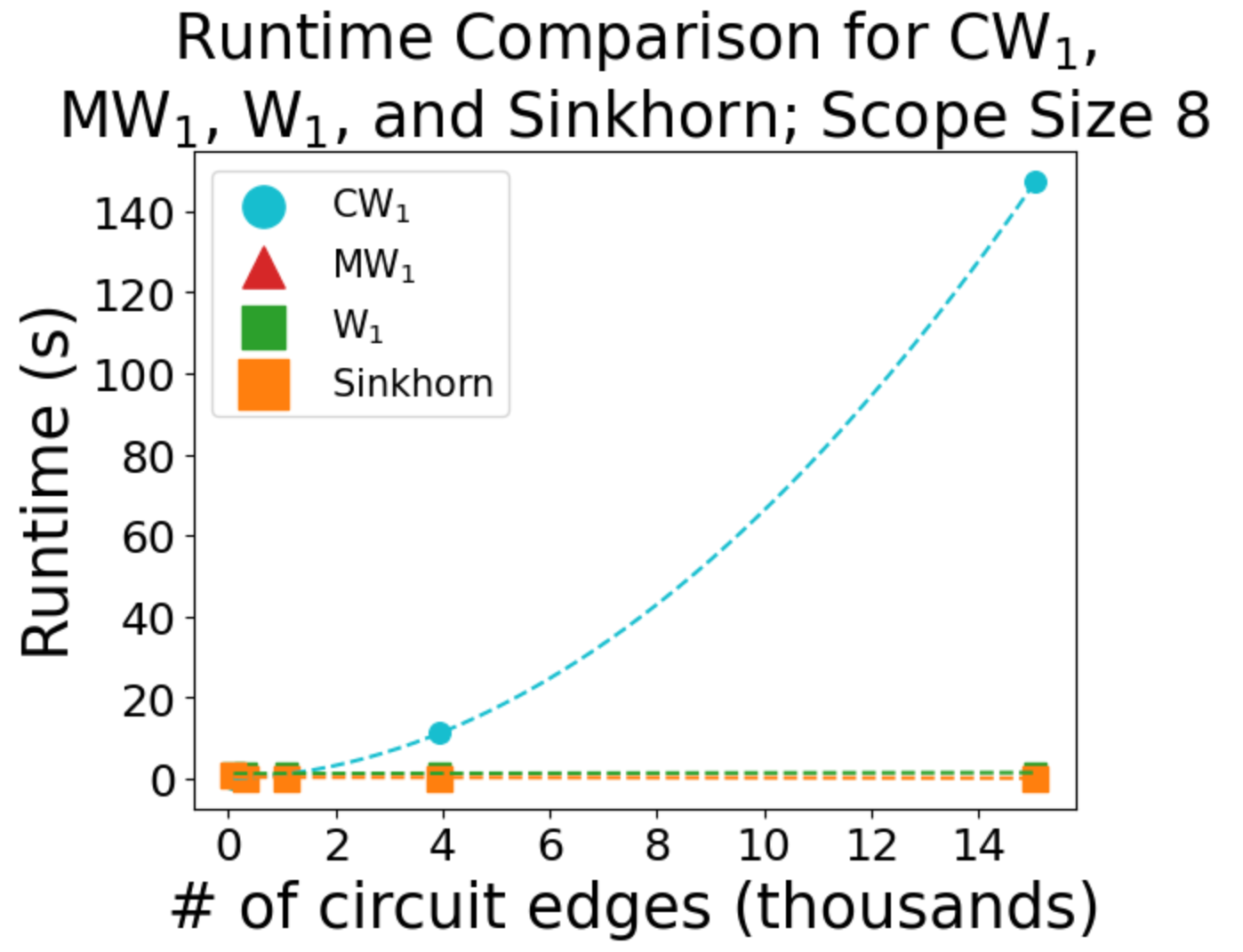} 
    \includegraphics[width=0.35\linewidth]{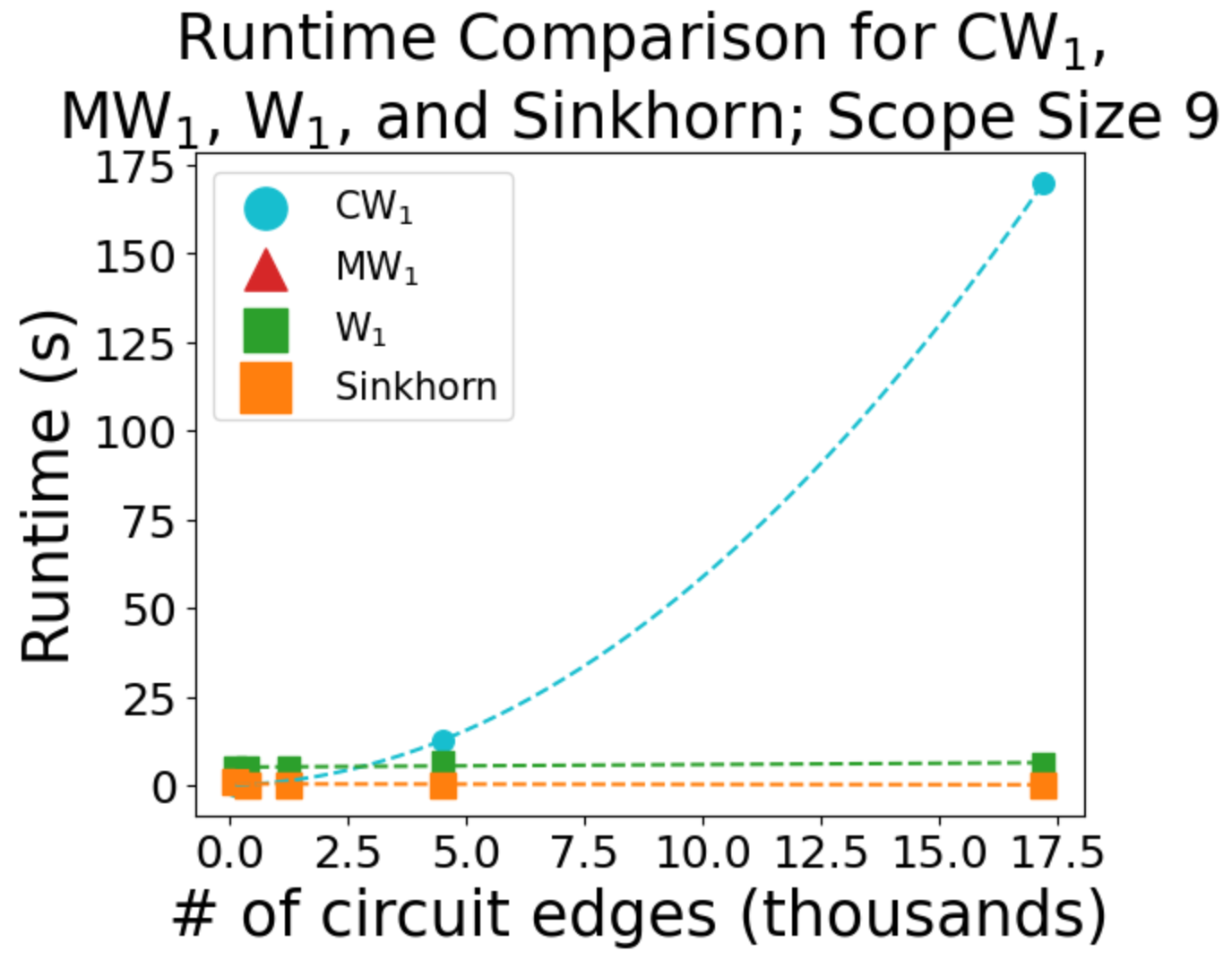} 
    \caption{Runtime for algorithms computing $\circDist_1$, $\gmmDist_1$, and $\wassDist_1$. For each plot, the circuit scope size $v$ is kept constant and the block size $k$ is varied to adjust the number of circuit edges. The ``number of circuit edges`` is the number of edges in one of the two circuits.}
    \label{fig:variableblocksize}
\end{figure}

\begin{figure}
    \centering
    \includegraphics[width=0.35\linewidth]{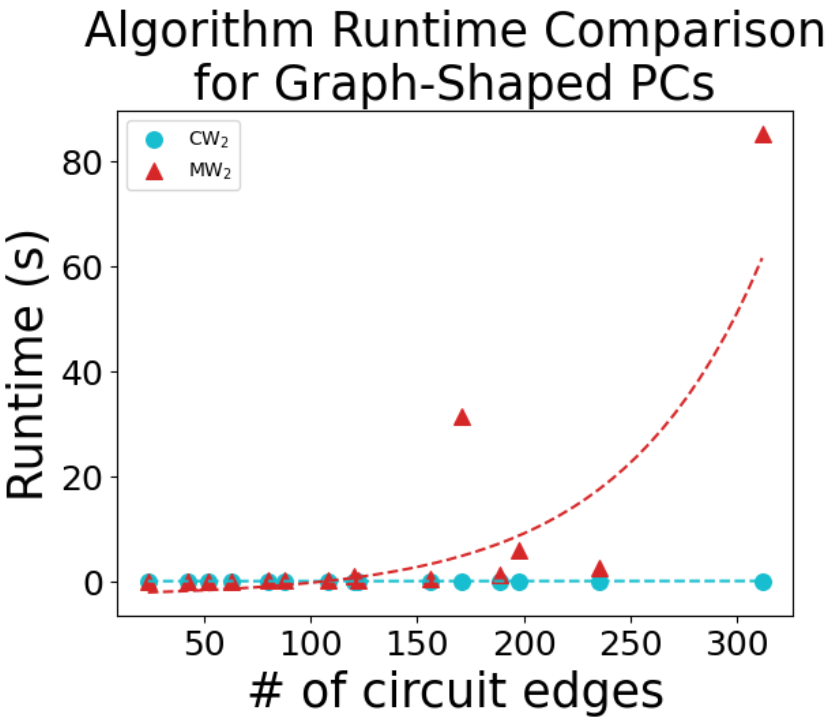} 
    \includegraphics[width=0.35\linewidth]{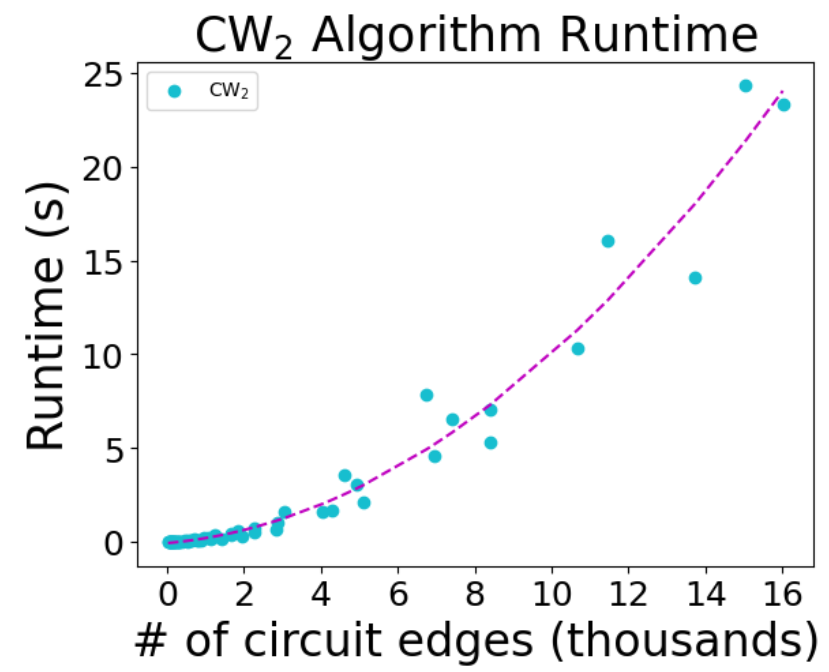} 
    \caption{Runtime for algorithms computing $\circDist_2$ and $\gmmDist_2$ between circuits with Gaussian input distributions. As computing $\gmmDist_2$ quickly becomes impractical and runs out of memory, we continue with larger circuits in the second figure.}
    \label{fig:gaussiancircuits}
\end{figure}

The value obtained for each circuit size and variable scope size is averaged over 20 runs, and we omit data points for experiments that ran out of memory. Lastly, all experiments were ran on a machine with an Intel Core i9-10980XE CPU, 256Gb of DDR5 RAM, and a single RTX3090Ti. To solve the linear programs we used Gurobi~\citep{gurobi}, a commercial linear program solver available under academic license.

We conduct runtime experiments by randomly generating 20 pairs of circuits with a given variable scope size $v$ and sum node branching factor---also referred to as block size for HCLT structures---$k$. We then compute $\circDist_1$, $\gmmDist_1$, and $\wassDist_1$ as described in Section~\ref{sec:experiments}. Data points not displayed either run out of memory or are infeasible due to numerical stability issues. See Figure~\ref{fig:variablescopesize} for a breakdown by circuit block size and Figure~\ref{fig:variableblocksize} for a breakdown by circuit scope size.

Lastly, we demonstrate the utility of our algorithm by computing $\circDist_p$ between PCs with Gaussian input distributions---for which $\wassDist_p$ cannot be computed simply using a linear program as done previously. See Figure~\ref{fig:gaussiancircuits} for the results. Note that the runtime of our algorithm is quadratic in the circuit size, while computing $\gmmDist_2$ has exponential runtime in the circuit size.

\subsection{Empirical Wasserstein Parameter Estimation Experimental Results}\label{sec:parameterlearningexps}
We investigated the computed Wasserstein objective and bits-per-dimension (BPD) of circuits of various sizes learned using EM and WM (our method). We found that larger circuits trained via EM have a significantly lower BPD than smaller circuits, which was not the case for circuits trained via WM. Looking at the Wasserstein objective for these circuits, we see that bpd is not directly correlated with the Wasserstein objective; circuits with a lower Wasserstein objective can have a slightly higher bpd, and vice versa.

Lastly, we consider a modification of Algorithm \ref{alg:minwass} that employs \emph{stochastic routing} of data at sum nodes; succinctly, we introduce hyperparameter $p$ that introduces a probability $p$ that a given data point is routed randomly with uniform probability to any given child node, and a probability $1-p$ that the data point is routed optimally as detailed in Algorithm \ref{alg:minwass}. When $p=0$, we refer to this as \emph{deterministic WM}; otherwise, we refer to the algorithm as \emph{stochastic WM}.

In our experiments, we found that $p=0.1$ yields the best results for minimizing the Wasserstein objective. For circuits of block size 4, we observe that this significantly decreases the Wasserstein objective without a significant change to the bits-per-dimension of the learned circuit. Over 5 random restarts, the stochastic WM algorithm resulted in a Wasserstein objective between 29947 and 29986; conversely, the deterministic WM algorithm resulted in a Wasserstein objective of 32766. However, this decrease in Wasserstein distance resulted in no decrease in bits-per-dimension for the trained models, with stochastic WM yielding circuits with BPDs of between 1.503 and 1.537. See Table \ref{table:emvswm} for more details.

\subsection{Additional Color Transfer Experiments}\label{sec:morecolortransfer}

\begin{figure}
    \centering
    \includegraphics[width=0.45\linewidth]{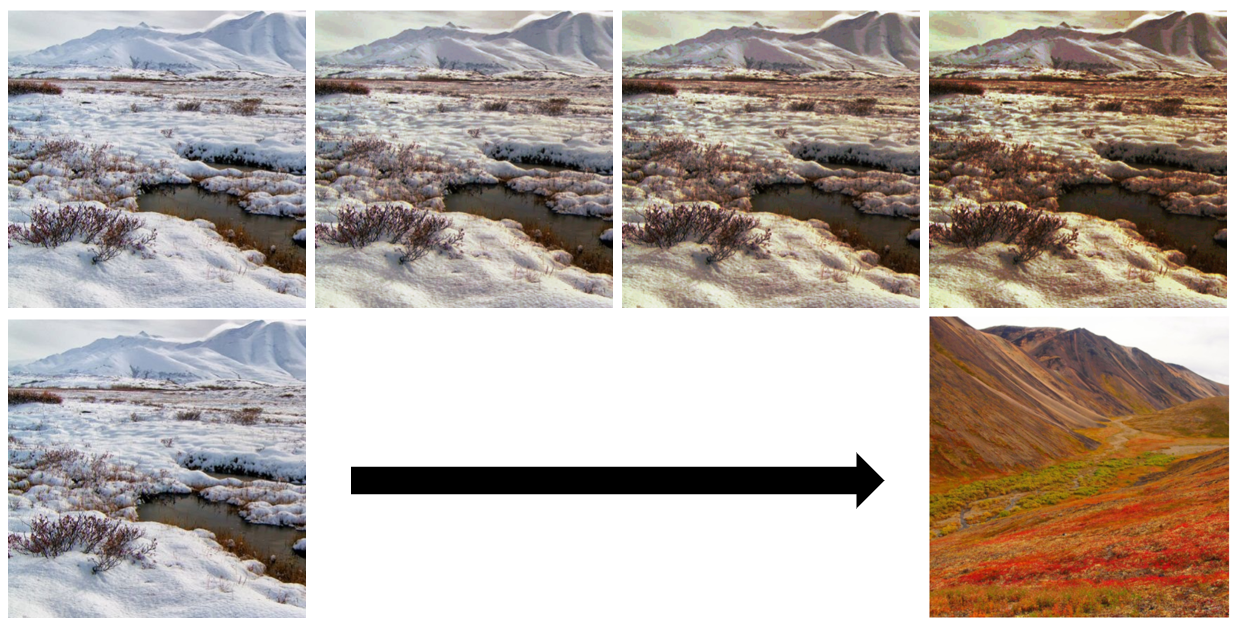}
    \includegraphics[width=0.45\linewidth]{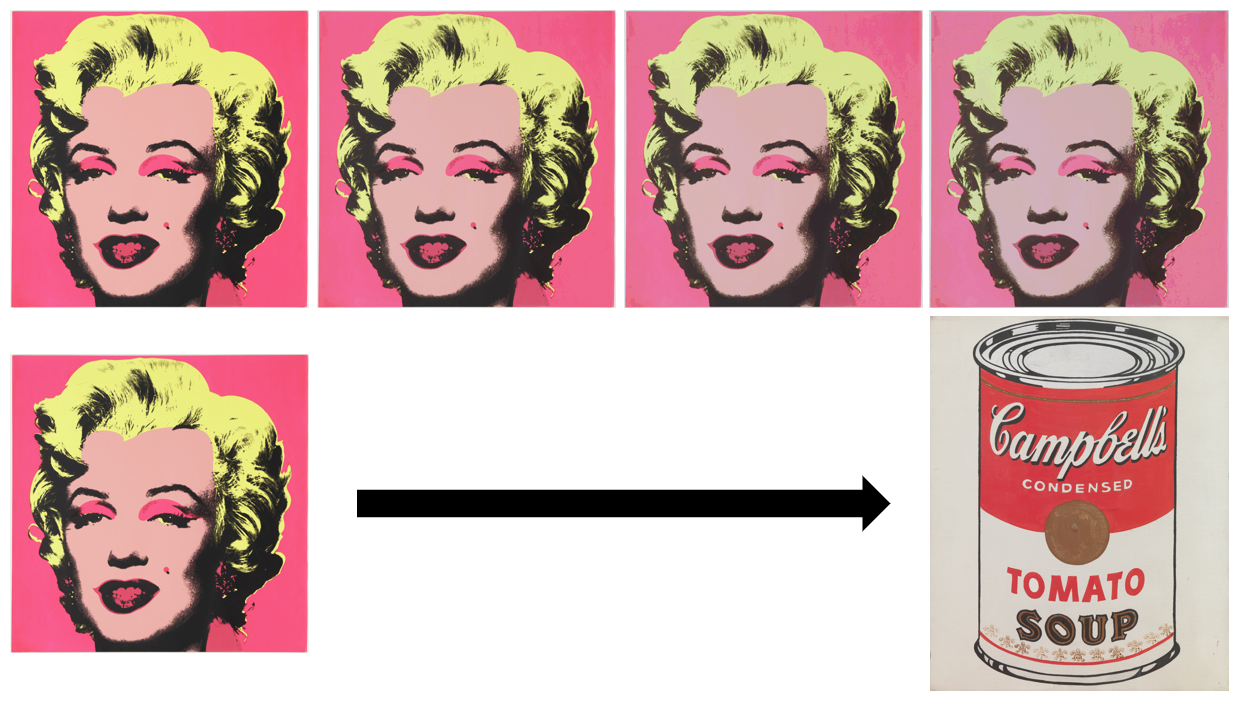}
    \includegraphics[width=0.45\linewidth]{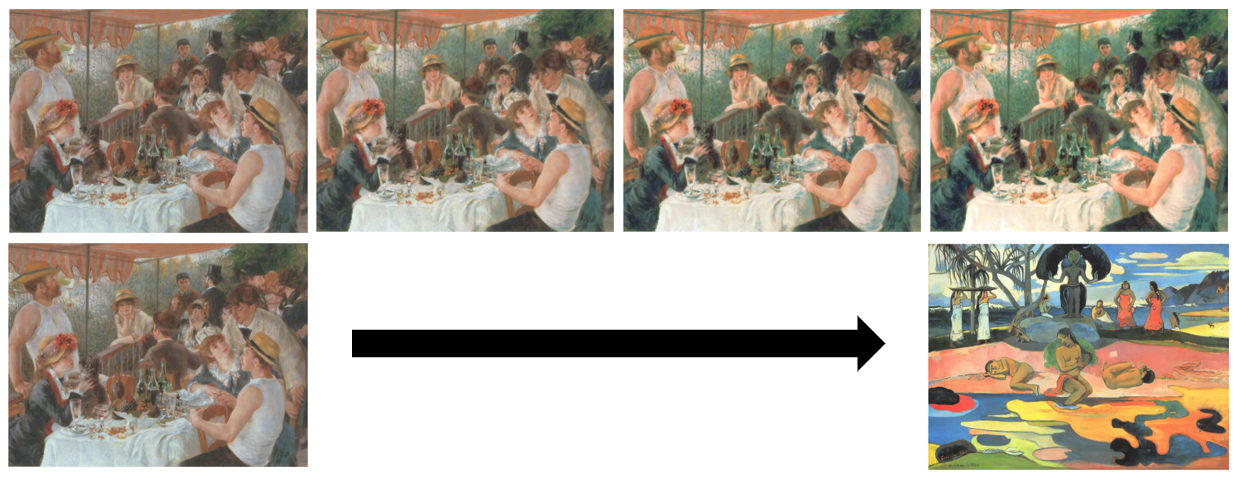}
    \includegraphics[width=0.45\linewidth]{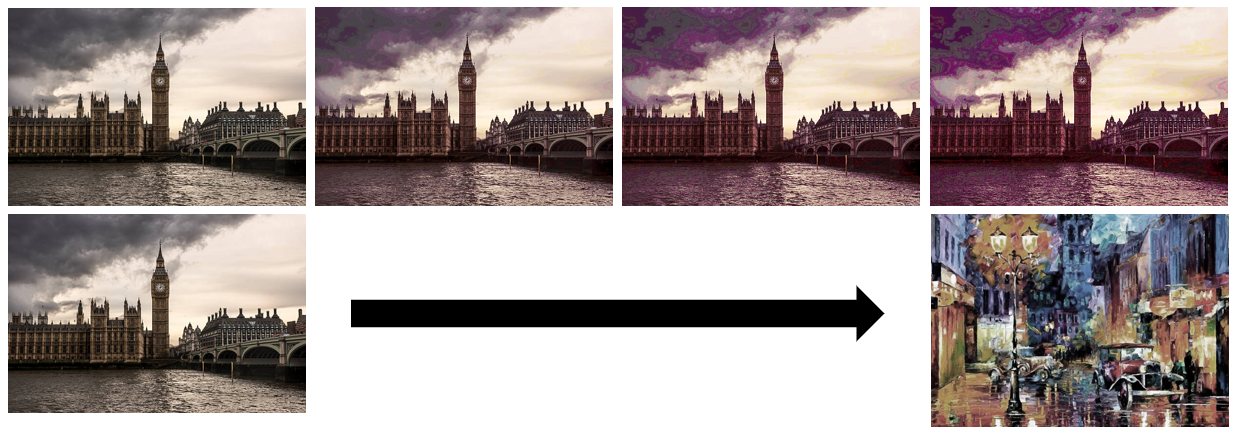}
    \caption{\small Color transfer between images along geodesics using coupling circuits, for $t=0, \frac{1}{3}, \frac{2}{3}$, $1$ in the direction of arrows.}
    \label{fig:morecolortransfer}
\end{figure}

An application of coupling circuits we explore in Section~\ref{sec:experiments} is color transfer via optimal transport. Succinctly, we transport the \textit{color histogram}---the 3-dimensional probability distribution of pixel color values---of image $a$ to that of another image $b$ by learning compatible PCs $P(\X)$ and $Q(\Y)$ over the color distributions of images $a$ and $b$, computing the optimal coupling circuit $C(\X,\Y)$, and transporting each pixel with color value $\x$ to the corresponding pixel $\y=\Ex_C[\Y|\X=\x]$ (which can be computed tractably). See Figure~\ref{fig:morecolortransfer} for some examples.

\subsection{Visualizing Transport Plans Between PCs}\label{sec:transportplanviz}

Since our algorithm does not only return $\circDist_p$ between two circuits but also the corresponding transport plan, we can visualize the transport of point densities between the two distributions by conditioning the coupling circuit on an assignment of random variables in one circuit. We can similarly visualize the transport plan for an arbitrary region in one PC to another by conditioning on the random variable assignments being within said region.

Since the transport plan for a single point (or a region of points) is itself a PC, we can query it like we would any other circuit; for example, sampling a set of corresponding points, as well as computing \textit{maximum a posteriori}---which is tractable if the original two circuits are marginal-deterministic \citep{probcirc}---for the transport plan of a point corresponds to the most likely corresponding point in the second distribution for the given point. Because a coupling circuit inherits the structural properties of the original circuit, it is straightforward to understand what queries are and are not tractable for a point transport map.

In Figure \ref{fig:transmaps}, we provide an example of visualizing the optimal transport plan between two randomly-generated PCs. We note that despite the transport plan being constrained to be a PC with a certain structure, the resulting transport plan matches our intuition as to what an optimal transport plan should look like.

\begin{figure}
    \centering
    \includegraphics[width=0.25\linewidth]{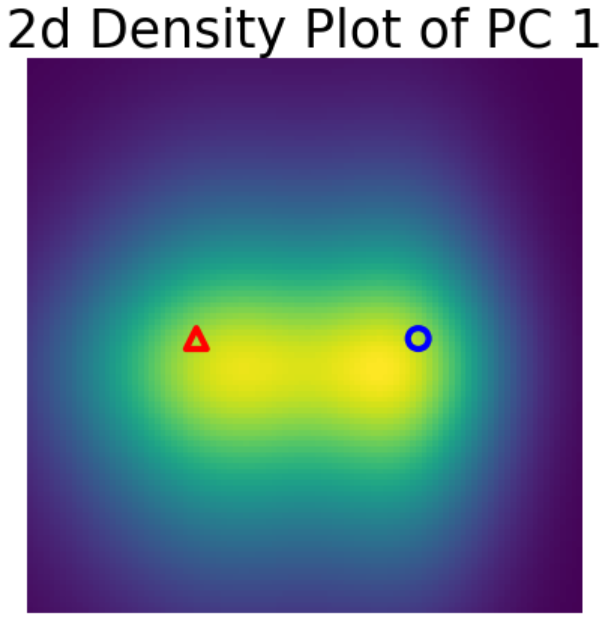} 
    \hspace{10pt}
    \includegraphics[width=0.25\linewidth]{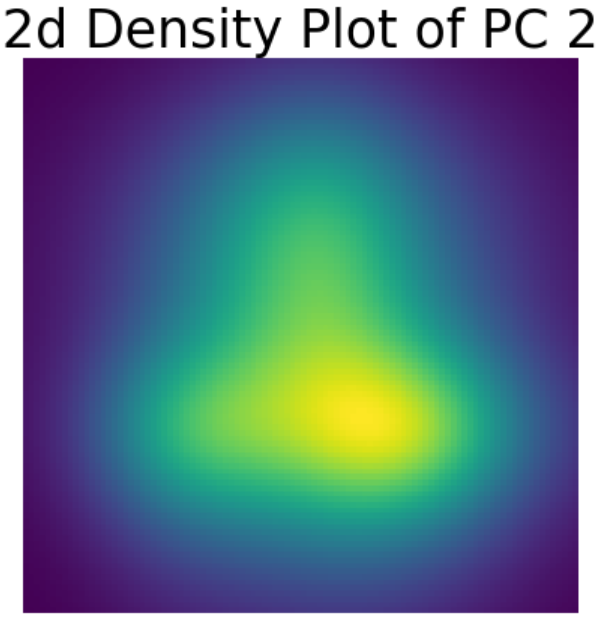}
    \hspace{200pt}
    \includegraphics[width=0.25\linewidth]{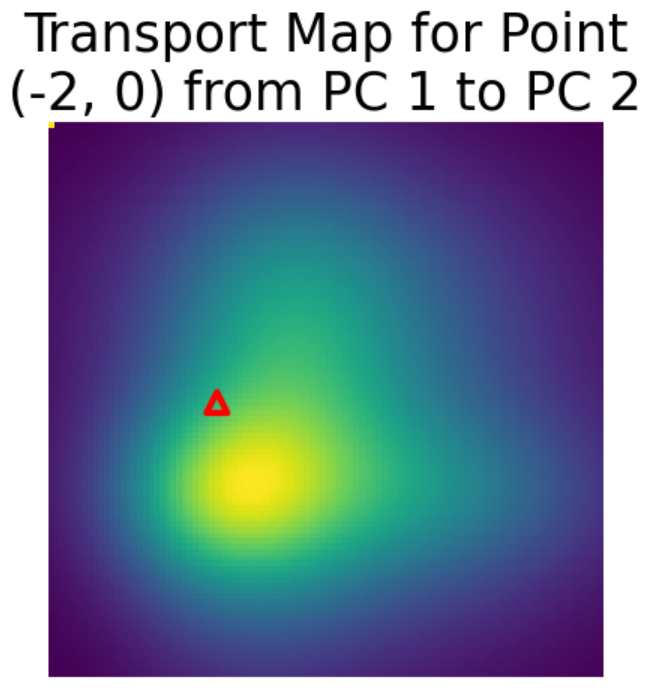} 
    \hspace{10pt}\includegraphics[width=0.25\linewidth]{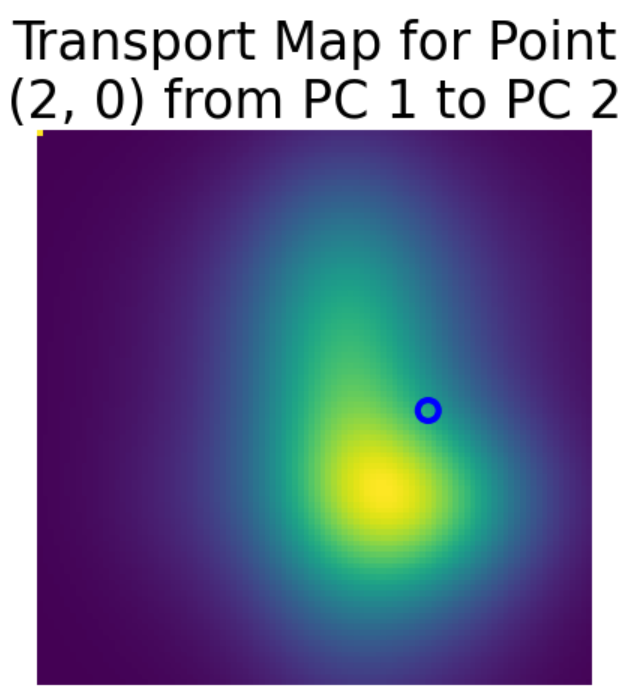}
    \caption{Visualization of transporting the indicated points from the distribution parameterized by PC 1 to the distribution parameterized by PC 2. The points (red triangle and blue circle) were arbitrarily selected to show how a point mass is redistributed according to the computed transport map. The top two figures visualize the input distributions, while the bottom two figures visualize where the point density indicated is transported to from the first to the second distribution.}
    \label{fig:transmaps}
\end{figure}

\end{document}